\documentclass{article}  % Comment this line out if you need a4paper

\usepackage{cite}
\usepackage{amsmath}
\usepackage{amsfonts}
\usepackage{amssymb}
\usepackage{graphicx}
\usepackage{comment}
\usepackage{xcolor}
\usepackage{bm}
\usepackage{subcaption}
\usepackage{algorithm}
\usepackage{algorithmic}
\newcommand{\RR}{\mathbb{R}}
\newcommand{\KK}{\mathbb{K}}

\newcommand{\TT}{\mathbb{T}}
\newcommand{\NN}{{\mathbb{N}}}

\newcommand{\knl}{{\mathfrak{K}}}
\newcommand{\rknl}{{\mathfrak{r}}}
\newcommand{\iH}[2]{{\left (#1,#2\right )_{H}}}

\newcommand{\RH}{{\mathcal{R}}}

\newtheorem{theorem}{Theorem}
\newtheorem{proof}{Proof}
\newtheorem{example}{Example}

\begin{document}

\title{Koopman Methods for Estimation of Animal Motions over
Unknown Submanifolds}

\author{Nathan Powell\thanks{ Department of Mechanical Engineering, Virginia Tech, Blacksburg. VA email: {nrpowell@vt.edu, boweiliu@vt.edu, kurdila@vt.edu}} \and Roger Liu \footnotemark[1] \and Jia Guo\thanks{ Department of Mechanical Engineering, Geogia Tech, Atlanta, GA. email: {jguo@gatech.edu}} \and Sai Tej Parachuri \thanks{Department of Mechanical Engineering and Mechanics, Lehigh University, Bethlehem, PA.  
email: {saitejp@lehigh.edu}} \and Andrew Kurdila \footnotemark[1] }

\maketitle

\begin{abstract}
This paper introduces a data-dependent approximation of the forward kinematics map for certain types of animal motion models. It is assumed that motions are supported on a low-dimensional, unknown  configuration manifold $Q$ that is regularly  embedded in a high dimensional Euclidean space $X:=\RR^d$.  This paper introduces a method to estimate a forward kinematics map from the unknown configuration submanifold $Q$ to an $n$-dimensional Euclidean space $Y:=\RR^n$ of observations. A known reproducing kernel Hilbert space (RKHS) is defined over the ambient space $X$ in terms of a known kernel function, and computations are performed using the known kernel defined on the ambient space $X$. Estimates are constructed using a certain data-dependent approximation of the Koopman operator defined in terms of the known kernel on $X$. However, the rate of convergence of approximations is studied in the space of restrictions to  the unknown manifold $Q$. Strong rates of convergence are derived in terms of the fill distance of samples in the unknown configuration manifold $Q$, provided that a novel regularity result holds for the Koopman operator. Additionally, we show that the derived rates of convergence can be applied in some cases  to estimates generated by the extended dynamic mode decomposition (EDMD) method. We illustrate characteristics of the estimates for simulated data as well as samples collected during motion capture experiments.

\end{abstract}

\section{Introduction}
Geometric methods that model the dynamics of complex multibody systems, such as those representing the motion of animals, are now well-established. \cite{bulloLewis2004geometric,lynch2017modern}. These techniques have been used in many cases to study the motion of an animal that is modeled as a collection of rigid bodies defined by bones that are interconnected by ideal joints. In these formulations, motions take values in certain types of smooth manifolds, and forward kinematic maps are defined over these manifolds. One attractive feature of these formulations is that such formulations are coordinate-free and intrinsic. \cite{bulloLewis2004geometric}  Still, for many animal motion studies, the dimension of the full configuration manifold can be large. The study of human motion, bat motion, or the flight of birds are just a few such cases \cite{Bender2015,bender2017gaussian}. For these animals, as well as others, it is often important to be able to identify submanifolds that provide accurate representations of particular motion regimes. 

This paper extends our most recent work in \cite{powell2020learning}. 
In the paper \cite{powell2020learning}, the authors explore methods to generate approximations of kinematics maps over certain smooth manifolds in applications that estimate the motion of reptiles. The analysis in \cite{powell2020learning} determines error bounds that closely resemble the more common case over Euclidean domains, but for the estimation of the functions over the motion manifolds.  However, in \cite{powell2020learning} it is assumed that a  known ``template manifold'' underlies a given motion regime.  This paper seeks to generalize these results to the case when the underlying motion manifold is unknown.

As we explain more fully below, this paper defines an estimation algorithm that does not require explicit knowledge of the unknown manifold. However, despite the fact that the manifold is unknown, strong direct theorems of approximation are derived that tell how fast estimates converge over the unknown manifold. The results are derived using Koopman theory, a popular operator-theoretic approach to the study of nonlinear dynamic systems. The analysis in this paper has the added benefit that it strengthens or refines convergence results in some recent papers on the approximation of Koopman operators.

In this paper, we assume that the motion of the animal is governed by the uncertain  equations
\begin{align}
    x_{i+1} &= f(x_{i}),\label{eqn:dynamics1}\\
    y_{i+1} &= G(x_{i+1}),
    \label{eqn:dynamics2}
\end{align}
where $X:=\RR^d$ is the state space, $Y:=\RR^n$ is the output space, the function $f:X\rightarrow X$ that determines the dynamics is unknown, and the observable function $G:X\rightarrow Y$ is  unknown. It is further assumed that this evolution law exhibits some additional ``hidden'' structure. Specifically, it is assumed that there is a compact, smooth Riemannian manifold $Q$ that is regularly embedded in $X$ that  is invariant under the flow generated by Equation  \ref{eqn:dynamics1}. The task is to collect  finite samples $\{(x_{k_i},y_{k_i})\}_{i=1}^M \in Z = X \times Y$ of state and output pairs at discrete times $\TT_M:=\{k_i\in \NN \ | \ 1\leq i\leq M\}$ along a trajectory  and use them to build an estimate of the observable function $G$. Here $\NN$ is the collection of nonnegative integers.  It is possible to choose the $k_i$ sequentially so that $k_1 = 1, k_2 = 2, \ldots $ etc, but the approach in this paper does not require uniform sampling in time. Ultimately, convergence of estimates of the forward kinematics map depends on the fact that the samples $\{x_{k_i}\}_{i=1}^M$ ``fill out'' $Q$ as $M\to \infty$, whether they are samples uniformly in time or not.

The analysis of the system in Equations \ref{eqn:dynamics1} and \ref{eqn:dynamics2} is carried out using the Koopman operator.  Recall that for this system the Koopman operator $U_f$ is defined via the  composition operator
$$
U_f g:=g\circ f 
$$
for any function $g$ in some suitable family. 
The Koopman operator can be used in a number of different ways to understand both the qualitative and quantitative behavior of nonlinear dynamical systems. A more careful review of relevant recent research along these lines is given in Section \ref{sec:relatedresearch} below. For now  we just review how the system above can be related to forecasting problems expressed in terms of the Koopman operator.  The study of forecasting problems is one of the more direct applications of  Koopman operators: see reference \cite{giannakis2020} for a detailed account of forecasting using the Koopman operator in continuous time.  

Suppose that $\{S(t)\}_{t\geq0}$ is a semigroup on the state space $X$ that defines a dynamical system. By definition the trajectory of the  system induced by the semigroup  is given by 
\begin{equation}
    x(t+h) = S(h)x(t)
\end{equation}
for each $t\geq 0$ and $h\geq 0$.  Correspondingly, the outputs  $y(t+h)\in Y$ at time $t+h$  are defined as 
\begin{equation}
    y(t+h) = G(x(t+h)). 
\end{equation}
The problem of forecasting in this continuous time setting is to predict the future output $y(t+h)$ at the  time $t+h$ using knowledge of the state $x(t)$ at time $t$. The problem in continuous time can be simplified by discretizing the system at times $t_i = ih$. By setting $x_i := x(ih)$ and $y_i := y(ih)$, we have  
\begin{align}
    x_{i+1} &= S(h)x_i \label{eq:dyn1} \\
    y_{i+1} &=G(x_{i+1}) = (G\circ f)(x_i) = (U_f G)(x_i). \label{eq:dyn2}
\end{align}
By choosing $f(x):=S(h)x$, the discretized forecasting problem has the form of our governing Equations \ref{eqn:dynamics1} and \ref{eqn:dynamics2}. 
The above question makes clear that if we are able to approximate the action of the Koopman operator $U_f$ on any $g$, we can use the approximation to obtain a forecast $y_{i+1} = (U_f G)(x_i)$ of the future output from knowledge of the current state.
 Here we also note that the $i^{th}$ step of the discrete flow in Equation \ref{eqn:dynamics1} is given by
$$
x_{i+1}:= ((U_f)^{i}f)(x_0). 
$$
Thus, obtaining an approximation of the action $U_f$ on $f$ can enable forecasting for a future state of the system. In this sense it is possible to interpret approximation of $U_f$ as solving a problem of system identification. 

The general strategy outlined above is quite practical. Examples \ref{ex:ex1} and \ref{ex:ex2} described below show that the model class described above is very rich: there are a wide variety of  models of animal motion that have the form above  that are obtained when geometric methods \cite{bulloLewis2004geometric,lynch2017modern} are discretized in time. 

\begin{example}
\label{ex:ex1}
{This first example illustrates how quite general motion estimation problems of the type governed by Equation \ref{eqn:dynamics1} may arise. For a wide variety of animal models, it is popular to create models that consist of rigid bodies connected by ideal joints. Such a model for human motion can be constructed by choosing as states the rigid body motion of a core body, spherical joints for the hips and shoulders, and revolute joints for all the remaining joints of the body. For human motion, the full configuration manifold might be taken as} 

\begin{align*}
\hat{X}:=\underbrace{SE(3)}_{\text{core body}} &\times \underbrace{SO(3)\times S^1 \ldots S^1}_{\text{left arm}} \times \underbrace{SO(3)\times S^1 \ldots S^1}_{\text{right arm}} \\
&\times  \underbrace{SO(3)\times S^1 \ldots S^1}_{\text{left leg}} \times \underbrace{SO(3)\times S^1 \ldots S^1}_{\text{right leg}}.
\end{align*}

{\noindent In the above expression $SE(3)$ and $SO(3)$ designates the special Euclidean group and special orthogonal group,  respectively, in three dimensions. The symbols $S^1$ is the circle, or one-dimensional torus. Since these are all compact Riemannian manifolds, it is known by the Nash embedding theorem that $\hat{X}\subset X=\RR^d$ for a $d$ large enough. 
Standard geometric methods as in \cite{bulloLewis2004geometric,lynch2017modern} describe how equations of motion generate a flow $\{S(t)\}_{t\in \RR^+}$ in continuous time over $\hat{X}$. By discretizing the flow at times $t_i:=ih$ with $h$ a fixed time step, we obtain the evolution law $x_{i+1}=S(h) x_i$. The discretization yields  dynamics of the form in Equation \ref{eqn:dynamics1} with $f=S(h):\hat{X}\subset X \rightarrow \hat{X}\subset X$. 
If we suppose that the output space $Y$ consists of the inertial positions of $j$ distinguished points on the body, the kinematics map $G: \hat{X} \rightarrow \RR^{3j}:=Y$. In motion capture studies that seek to study certain motion regimes, it is then natural to hypothesize an existence of an unknown submanifold $Q\subseteq \hat{X}\subseteq X$. At the discrete times $\TT_M:=\{k_i\ | \ 1\leq i\leq M\}$, a set of sample states $\Xi_M:=\{\xi_{M,i}:=x_{k_i} \ | \ 1\leq i\leq M\}$ and outputs $\Upsilon_M:=\{\upsilon_{M,i}:=y_{k_i+1} \ | \ 1\leq i\leq M\}$ are collected along trajectories of the uncertain system. 
The estimation problem is then to use the measurements $\{(\xi_{M,i},\upsilon_{M,i})\}_{i=1}^
M$ to build approximations of $U_f G:Q\rightarrow Y$, where manifold $Q$ is unknown. The strategy in this paper can be realized, at least in principle, using RKH spaces with bases that are the product of bases for functions over $SO(3)$ or other homogeneous manifolds. Such bases are studied, for example, in \cite{hangelbroek2011surface,hangelbroek2012polyharmonic}. }
\end{example}

\begin{example}
\label{ex:ex2}
In this paper, we perform numerical studies of some simpler problems that are common in studies of animal motion. Instead of modeling or studying motions in the full configuration space above, it is often of interest to understand the motion for just an appendage, or across a joint. Such studies can be used in the design of prosthetics. For simplicity, we carry out numerical studies in this paper when $Q$ is a unknown Riemannian manifold that is regularly embedded in $X:=\RR^d$. Again, if we assume that $Y$ is defined in terms of the inertial positions of $j$ points on the body, we have $G:X\rightarrow Y$ with $Y:=\RR^{3j}$. The estimation task is to use samples $\{(\xi_{M,i},\upsilon_{M,i})\}_{i=1}^
M \subset X\times Y$ to build approximations of $G$, even though the manifold $Q$ is unknown. 
\end{example}

\subsection{Related Research}
\label{sec:relatedresearch}
Overall, the recent research that is most relevant to this paper includes studies of formulation of animal motion, machine learning theory, and Koopman operator theory. We examine how each category relates to this paper in turn. 

A theoretical foundation for the formulation of models of animal motion can be found in \cite{bulloLewis2004geometric,lynch2017modern}.  
In addition, there are a wide variety of kinematic studies that  evaluate specific animal motions \cite{Birn-Jeffery2016,Ingersoll2018,porro2017inverse}. Some notable recent studies have focused on discovering fundamental principles of locomotion \cite{aguilar2016review} and lower-dimensional representations of motion models across a vast number of different animals \cite{catavitello2018kinematic}. However, it seems that  applying learning algorithms to obtain kinematic models has been primarily applied to human motion \cite{brand2000style,wang2007gaussian}. Recent learning methods have utilized neural networks and deep learning to synthesize and predict human motion \cite{butepage2017deep,jain2016structural}. To address the time-dependent nature in synthesizing human motion, recurrent neural networks (RNNs) have been implemented  \cite{fragkiadaki2015recurrent,li2017auto}. To the authors' knowledge, none of these efforts case the animal motion estimation problem in terms of Koopman theory, as is done in this paper.

Some parts of this paper extend certain problems of machine or statistical learning theory to the case where we want to estimate functions (kinematics maps) over manifolds.  A theoretical background on distribution-free  learning theory over Euclidean spaces can be found in \cite{gyorfi2006distribution,devore2004mathematical,devore2006approximation}. Versions of learning theory cast over RKH spaces are also relevant, particularly when we want to extend the learning problems to some manifold. The approaches in \cite{williams2006gaussian} and \cite{scholkopf2002learning}, that are formulated in Euclidean Spaces, are a good starting point in this case.

Recent research in Koopman theory often synthesizes the viewpoints of learning theory and dynamical systems theory. Koopman operator theory has been one of the primary tools that is used to pose learning problems over time-dependent data, especially with strong nonlinear dynamics are at play \cite{mezic2020koopman}. Excellent texts such as \cite{lasota2013chaos,mezic2015applications,brunton2019data} as well as a number of survey articles including \cite{brunton2021modern} give ample motivation for the study and use of Koopman theory.     As observed in \cite{budivsic2012applied}, with increases in computational power, the proliferation of experiments that yield high-dimensional data, and breakthroughs in the study of linear dynamics, there has been a surge of renewed interest in the Koopman operator and its spectral decomposition over the last decade.  An extensive overview of the spectral properties of both the discrete and continuous Koopman operator can be found in \cite{brunton2019data}, \cite{brunton2021modern}. 

As explained in  studies on Koopman theory such as \cite{brunton2019data,budivsic2012applied,korda2018convergence}, perhaps the most compelling theoretical reason for the use of Koopman operator theory is the trade-off it enables from the study of a nonlinear system to an linear system. Most frequently, the approach is used when the original system as in Equation \ref{eqn:dynamics1} is nonlinear and finite-dimensional. Via Koopman theory, the system governing the evolution of observables in Equation \ref{eq:dyn2} is linear,  as is the Koopman operator $U_f$. However, the linear system that governs the dynamics of the observables is infinite-dimensional.  Consequently, Koopman theory often requires a careful consideration of practical finite-dimensional approximations of the infinite-dimensional operator. Many recent studies focus on the convergence of approximating the Koopman operator or quantities associated with the Koopman operator such as the Koopman modes \cite{budivsic2012applied} \cite{rowley2009spectral}.

Convergence studies of Koopman theory have also been related to dynamic mode decomposition  (DMD) algorithm\cite{tu2013dynamic}. A number of studies have used the Koopman operator framework to examine the rates of convergence of EDMD method \cite{williams2015data} \cite{klus2016} to the Galerkin approximation of the Koopman operator. Additional convergence studies include those based on the infinitesimal generator \cite{klus2020} of the Koopman operator, as well as convergence of eigenfunctions of the Koopman operator \cite{korda2018convergence}. The behavior of these eigenfunctions under conjugate transformations are examined in \cite{mezic2020numerical}. As powerful and promising as these methods are, none of these references study rates of convergence of the estimates. 

The primary difference between the approach taken in this paper to these recent efforts like \cite{brunton2019data,budivsic2012applied,korda2018convergence} is that we utilize the structure of direct approximation theorems to describe rates of convergence of approximations to the Koopman operator. To clearly explain the contributions of this paper, we review the structure of direct approximation theorems in Section \ref{sec:direct_thms}. A careful description of convergence of approximations of the Koopman operator in recent efforts \cite{giannakis2020,williams2015data,budivsic2012applied,klus2020,mezicbook} is given in Section \ref{sec:approx_koopman}. Finally, these two subsections are employed in Section \ref{sec:overview} to be precise about the nature of approximation theorems for Koopman operators in the recent literature.
\subsubsection{Direct Approximation Theorems}
\label{sec:direct_thms}

The convergence of approximations of the Koopman operator in the recent literature has been studied in the uniform, strong, and weak operator topologies. We only discuss the uniform and strong operator topologies in this paper as they seem most relevant to the approach in this paper. 

Let $(W,\|\cdot\|_W)$ be some normed vector space and denote by $\cal{L}(W)$ the collection of bounded linear operators that map $W$ into itself. By definition, a sequence of bounded linear operators $\{T_n\}_{n\in \NN}$ converges in the uniform operator topology if it holds that 
$
\| T_n - T \|_\infty \rightarrow 0
$
as $n\rightarrow \infty$ where the operator norm $\|\cdot\|_\infty$ is defined in the usual way as 
$$
\|T\|_{\infty}:=\sup_{
w\in W, v\not = 0} \frac{\|Tw\|_W}{\|w\|_W}.
$$
The sequence $\{T_n\}_{n\in \NN}\subset \cal{L}(W)$ converges to $T\in \cal{L}(W)$ in the strong operator topology if for any $w\in W$ we have $\|T_nw-Tw\|_W\rightarrow 0$ as $n\to \infty$. 
The uniform operator topology is stronger (finer) than the strong operator topology since convergence of $\|T-T_n\|_\infty\rightarrow 0$ implies that $\|Tw-T_nw\|_W\leq \|T-T_n\|_\infty \|w\|_W \rightarrow 0$.
It is known that an operator $T$ is compact if and only if there is a sequence of finite dimensional  linear operators $\{T_n\}_{n\in \NN}\subset \cal{L}(W)$ that converge to $T$ in the operator norm. It is a rather special case that Koopman operators are compact, but this case is studied under certain circumstances, for  example in  references \cite{korda2018convergence}, \cite{klus2020a}, and \cite{kurdila2018koopman}, which are discussed more fully below. 

There is a  standard framework in approximation theory for characterizing the rates of convergence of approximations that can be used to study  the strong convergence of operators. This setup makes use of another linear space  $S\subseteq W$. We use the symbol $S$ since it has a stronger topology  than that on the weaker space $W$: it is assumed that $S$ is dense in $W$ and that there is a constant $C>0$ such that $\|w\|_W\leq C\|w\|_S$ for all $w\in S$. Approximations are constructed in terms of a family $\{W_N\}_{N\in \NN}$ that is a sequence of  finite dimensional subspaces whose union is dense in $W$. Suppose that $P_N:W\rightarrow W_N$ is a family of uniformly bounded mappings, each $P_N$ onto $W_N$.  We say that the pair of sequences $(W_N,P_N)_{N\in \NN}$ and the pair of spaces $(W,S)$ satisfies a direct theorem of approximation with approximation rate $N^{-r}$ if there is some other constant $\tilde{C}$ such that  we can write
\begin{align}
\label{eq:direct}
\|w-P_Nw\|_W\leq \tilde{C} N^{-r}\|w\|_S
\end{align}
for all $w\in S\subseteq W$. This inequality can be viewed as a statement that functions in the ``smaller, more regular'' space $S$ converge at a rate $r$ that characterizes such regularity.  Often in machine learning theory or approximation theory the membership in the space $S\subseteq W$ is referred to as a ``prior'' on a function $f\in W$: it is information that conveys that approximations of such functions converge at a certain rate.  The study of how priors  and approximation operators $P_n$ can be selected to yield direct theorems of approximation is one of the central problems of approximation theory. This problem has been studied in great detail, see \cite{devore1993constructive,wendland2004scattered} for a wide collection of choices of the spaces $W$ and $S$ over subsets of Euclidean Space. In this paper we use this construction to study the rates of convergence of approximations of the Koopman operator in the strong operator topology.  We apply the analysis above when $w$ in Equation \ref{eq:direct} is chosen to be $w:=U_f g$, with $U_f$ the Koopman operator of a discrete deterministic system. The rate of approximation $N^{-r}$ will be defined in terms of the fill distance of samples in the manifold.

\subsubsection{Approximations of Koopman Operators}
\label{sec:approx_koopman}
As mentioned above, one of the compelling reasons that motivates the use of Koopman operators is the fact that, while the original dynamics is determined by a nonlinear function $f$ that acts on a finite dimensional space, the Koopman operator is linear and acts on an infinite dimensional space of functions. Any practical use of a Koopman operator must address questions related to its approximation. 
Of the recent articles \cite{giannakis2020,korda2018convergence,klus2016,mezic2020numerical,klus2020a,klus2016,kurdila2018koopman} that study approximations of the Koopman operator under various working assumptions, perhaps the most relevant to this paper are \cite{korda2018convergence,giannakis2020,klus2020a,kurdila2018koopman}. We review each of these papers in some detail below to make clear the often subtle difference among the references and to clarify the contributions of this paper.

Reference \cite{korda2018convergence} is a detailed analysis of the convergence of approximations $U_{M,N}$ determined by the EDMD algorithm of the Koopman operator $U_f$. The EDMD algorithm and its relationship to Koopman theory is discussed in some detail in Section \ref{sec:edmd}. The integers $N,M$ denote the number of samples $M$ and the dimension $N$ of a  subspace of approximants $H_N$ used to build estimates. The convergence analysis is carried out in the Lebesgue space $L^2_\mu(Q)$ where $\mu$ is a measure on the manifold $Q$.  The paper studies carefully the relationship of approximations of the EDMD algorithm of approximants $H_n\subset L^2_\mu(Q)$. A few different types of convergence are studied. Let $P_{\mu,n}:L^2_\mu\rightarrow H_n$ denote the $L^2_\mu$-orthogonal projection of $L^2_\mu$ onto $H_n$, and define the projected operator  $U_n:=P_{\mu,N} U_f|_{H_N}$ with $U_f|_{H_N}$ the restriction to the space of approximants $H_N$.  The paper shows in Theorem 2 that   {\em for a fixed dimension $N$ of the approximant subspace $H_N$}, the EDMD approximations $U_{M,N}$ converge to  the  projected operator $U_N:=P_{\mu,N} U|_{H_N}$  as $m\rightarrow \infty$ in any operator norm (which are equivalent on $H_N$ due to the assumption that $N$ is fixed and finite). The paper also shows that $P_{N,\mu} U_f P_{\mu,N}\rightarrow U_f$ as $N\rightarrow \infty$ in the strong operator topology over $L^2_\mu(Q)$.  The paper \cite{korda2018convergence} does not address how the choice of different kernels influences or affects the choice of the space $H\subseteq L^2_\mu(Q)$, nor how the choice of such a prior determines  the rate of convergence of the approximations. 

An altogether different strategy is used to derive rates of convergence of certain types of Koopman operators in \cite{klus2020a}. In contrast to \cite{korda2018convergence} above, a regularized {\em kernel } Koopman operator $\cal{U}^\epsilon$ is introduced in terms of an integral operator that depends on covariance operators associated with the inputs and outputs. In this notation $\epsilon$ is the regularization parameter.  Using $M$ samples, regularized approximations of the Koopman operator $\mathcal{U}^\epsilon$, as well as the associated regularized Frobenius-Perron operators, are defined in terms of integral operators whose integrands are the  kernel of an RKH space defined over some compact set. Theorem 3.14 of \cite{klus2020a} derives rates of convergence of empirical estimates $\hat{\mathcal{U}}^\epsilon_M$ to the {\em regularized operators} ${\mathcal{U}}^\epsilon$.   It is shown that the operator norm error $\|\mathcal{U}^\epsilon - \hat{\mathcal{U}}^\epsilon_M\|=O_p(M^{-1/2}\epsilon^{-1})$ as the number of samples $M\rightarrow \infty$. Here the notations $O_p(\cdot)$ denotes the order in probability. In brief comments following the theorem, it is noted that this bound implies that the the regularization parameter should be chosen to be varying as the number of samples $M$ increases. As above, this  convergence guarantee does not explain choice of kernels  affects the  rate of convergence for   the choice of different kernels or priors. 

Reference \cite{giannakis2020} studies   kernel analog filtering (KAF) for certain types of forecasting problems. This paper considers deterministic dynamical systems that evolve in continuous time over a set $\Omega$.  Moreover, it is assumed that the dynamical systems are measure preserving. The kernel analog filtering problem seeks a function $f_h$ that minimizes
$
    \|f_h \circ X - {\cal{U}}^h Y \|_{L^2_\mu(\Omega)}
$
with $\mu$ the invariant measure of the underlying dynamical system and $h$ is a time step or lag time. Here the Koopman operator ${\cal{U}}^h$ acting on the output function $Y$ corresponds to the composition of $Y$ acting on a flow map $\Phi_h: X \to X$ over a time-step $h$. 
Approximations in the paper \cite{giannakis2020} are studied  in terms of eigenfunctions of a compact, self-adjoint operator associated with an integral operator having as integrand a kernel of an RKH space.
For theoretical purposes, subspaces of approximants are defined in terms of the span of a finite collection of $N$  eigenfunctions. Finite dimensional approximations of the function $f_h$ in the paper can be understood as an  approximation of a Koopman operator, see Equation (19) of that reference. 
The convergence of approximations of the Koopman operator are studied in the Lebesgue space $L^2_\mu(\Omega)$ where $\mu$ is the measure that is invariant under the semigroup that determines the dynamical system. Theorem 11 proves convergence of the estimates as  the dimension $N$ of the eigenspace used for approximations approaches infinity. In this sense Theorem 11 of \cite{giannakis2020} qualitatively  resembles the convergence result in Theorem 2 in  \cite{korda2018convergence}. Theorem 14 of \cite{giannakis2020} additionally shows that empirical approximations, based on $N$ eigenfunctions and $M$ samples, converge to a  regression function as $M,N\rightarrow \infty$. This conclusion is stronger than the analysis in \cite{korda2018convergence} above. However, this paper does not discuss how the choice of the kernel, and a set of priors, affects the rate of convergence.

In contrast to the previous studies, reference \cite{kurdila2018koopman} does examine rates of convergence using direct theorems of approximation of the Koopman operator under a few different situations. However, the  approximations are
generated by finite dimensional bases of wavelets, multi-wavelets, and eigenfunctions of the Koopman operator. Such constructions are only useful when the finite-dimensional bases are ``well adapted" to the dynamics under study. Cases such as these might include evolutions on a torus where periodized standard bases can be defined. 
The studies \cite{powell2020learning}\cite{powell2021distribution} examine the rates of convergence of the EDMD estimate to the Koopman operator by interpreting it as an empirical risk minimizer as examined in \cite{devore2004mathematical}. In the current paper, we consider smoother estimates generated from reproducing kernel Hilbert (RKH) spaces. This method is well-suited to taking scattered samples over unknown or irregular domains, or over subsets of manifolds.

\subsection{Overview of New Results}
\label{sec:overview}
To be precise about the contributions of this paper to the rapidly growing field of Koopman theory, we briefly summarize the overall strategy taken in this paper. This overall strategy is common to many of the approaches outlined in References \cite{budivsic2012applied,brunton2021modern,klus2016,korda2018convergence,giannakis2020,mezic2020numerical,klus2020a}, but, as mentioned in the introduction, we focus on determining how to ensure error bounds not found in these references. Since $Q\subset X$ is unknown, but the ambient space $X$ is known, a kernel function  $\knl:X\times  X\rightarrow \RR$ is selected that defines a reproducing kernel Hilbert space  (RKHS)  $H$ of real-valued functions over the known, ``large set'' $X$. A kernel section, or kernel basis function, centered at a point $x\in X$ is defined in terms of the kernel as $\knl_x(\cdot):=\knl(x,\cdot)$. Finite dimensional spaces of approximants $H_M$ are defined in terms of bases $\{\knl_{\xi_{M,i}}\}_{i=1}^M$, which  depend on some set of samples $\Xi_M:=\{\xi_{M,i}:=x_{k_i} \ | \ 1\leq i\leq M\}$. As mentioned before, these samples are collected along trajectories of the uncertain system at the discrete times $\TT_M:=\{k_i\ | \ 1\leq i\leq M\}$, which do not need be uniformly sampled in time. Hence, we have the spaces of approximants $H_M:=\text{span}\{ \knl_{\xi_{M,i}} \ | \ \xi_{M,i}\in \Xi_M\}$. Approximations of the observable function $G$ for the uncertain system are built using the finite dimensional spaces $H_M$ and the $H$-orthogonal projection $\Pi_M:H\rightarrow H_M$. Data-dependent approximations $U_f^M$ of the Koopman operator $U_f$ are defined by the identity $U_f^M g:=\Pi_M( (\Pi_M g)\circ f)$, which has a coordinate representation shown in Equation  \ref{eq:UMf_coord} in Section \ref{sec:data_driven}. By inspection of the coordinate expression, it is clear that the approximation $U_f^M$ of the Koopman operator $U_f$ can be constructed from the samples $\{(\xi_{M,i},\upsilon_{M,i})\}_{i=1}^
M$ and the choice of finite dimensional space $H_M$. This approximation does not require knowledge of the unknown function $f$. This paper introduces an approximation method in which the distribution of the samples $\{\xi_{M,i}\}_{i=1}^M$, the definition of the space $H_M$, and the choice of the kernel $\knl$ determines the rate of convergence of $U^M_fg\rightarrow U_fg$. It should be emphasized that we make no assumption of ergodicity, nor that the evolution preserves a fixed measure.

%\subsubsection{Contributions}
There are three distinct contributions in this paper to the state-of-the-art in estimation methods for the system in Equations \ref{eqn:dynamics1} and \ref{eqn:dynamics2}. 

\underline{(1) \it Reduced Order Models of Animal Motion:}  Although Koopman methods have been studied for a wide variety of application areas, this paper is the first systematic study of the method for characterizing animal motions over submanifolds. The approach should be of general interest to those studying reduced order models of animal motions. The paper describes a rigorous formulation of a general method for approximating such motions, and the theory includes estimates of the rates of convergence in various function spaces as the number of samples $M$ (equal to the dimension of the subspace of approximants) increases. Such estimates are not available in the literature on motion estimation such as in  \cite{brand2000style,wang2007gaussian,jain2016structural,butepage2017deep,fragkiadaki2015recurrent,li2017auto}.
\newline \underline{(2) \it Convergence Rates for Koopman Approximations:}  In addition, this paper descr- ibes general, strong rates of convergence of data-dependent for some approximations of the Koopman operator. This result is of interest in its own right. Full details on the derivation and interpretation of these rates is given in the discussion of Theorem \ref{th:approx_koopman} below, but we briefly summarize the structure and novelty of these rates here.   
The primary result of this paper defines a  regularity assumption in the framework of Koopman theory that is particularly powerful and enables determination of strong error bounds. The regularity assumption  states that if the pullback space $f^*(H):=\{g:Q\rightarrow \RR\ | \ g=h\circ f, h\in H\}$ is continuously embedded in a Sobolev space, then we get strong rates of convergence. These rates of approximation are bounded by $O(h_{\Xi_M,Q}^r)$ where $h_{\Xi_M,Q}$ is the fill distance of samples $\Xi$ in the manifold $Q$,
$$
h_{\Xi_M,Q}:=\underset{q\in Q}{\text{sup}} \ \underset{\xi_{M,i}\in \Xi_M}{\text{inf}} d_Q(q,\xi_{M,i})
$$
where $d_Q$ is the metric over $Q$ and $r$ is an index that measures smoothness in a Sobolev space. Note that this is a special case of the rate of convergence in Equation \ref{eq:direct} where $h^r_{\Xi_M,Q}$ is the same order as $N^{-r}$. Thus, even though the manifold $Q$ is unknown, we obtain a geometric characterization of the error. The convergence depends on how quickly samples fill up the unknown manifold. These rates for the Koopman approximation problem are qualitatively similar to  rates of convergence that are familiar in regression in Euclidean spaces \cite{gyorfi2006distribution}, interpolation over manifolds \cite{fuselier,hangelbroek2012polyharmonic}, Chapter 10 of \cite{wendland2004scattered}, and ambient approximation over manifolds \cite{maier2018}. None of the recent studies of Koopman approximations \cite{klus2016,korda2018convergence,giannakis2020,mezic2020numerical,klus2020a} derive such rates of convergence in terms of fill distances. 
\newline \noindent \underline{(3) \it Convergence Rates of the EDMD Algorithm} Finally, this paper defines a general data-dependent approximation of the Koopman operator as it acts on a variety of function spaces. In some instances, this data dependent operator can be interpreted as a specific implementation of the extended dynamic mode decomposition (EDMD) algorithm. With this identification it is possible to derive rates of convergence for the EDMD algorithm in certain types of function spaces. This result is described in Theorem \ref{th:edmd} and also has no precedent in the recent studies of the EDMD algorithm, as in \cite{klus2016,mezic2020numerical}. 

\section{Background Theory}
\subsection{Symbols and Nomenclature}
In this paper, $\RR$ is the set of the real numbers. We write $a \lesssim b$ to mean that there is a constant $c$ that does not depend on $a,b$ such that $a\leq cb$, and $\gtrsim$ is defined similarly. We let $X=\RR^d$ be the input space, and $Y:=\RR^n$ be the output space. We assume that $Q$ is a smooth, compact Riemannian manifold that is regularly embedded in $X=\RR^d$. The symbol $C(Q,Y)$ represents the space of continuous, $Y$-valued functions over the manifold $Q$, equipped with the usual supremum norm. We set $L^2_\mu(Q,Y)$ to be the usual Lebesgue space of $Y$-valued, $\mu$-integrable functions over the manifold $Q$, with $\mu$ a measure on the manifold. If the second argument is left out, it is understood that the output space is simply real-valued $C(Q)= C(Q,\RR)$, $L^2_\mu(Q)= L^2_\mu(Q,\RR)$, etc. This paper also employs the scalar-valued Sobolev space $W^{s,p}(Q)$ over the compact Riemannian manifold $Q$, which is defined intrinsically in \cite{hangelbroek2012polyharmonic} or extrinsically in \cite{fuselier}. Roughly speaking, for integer $s>0$ this space consists of all functions in $L^p_\mu(Q)$ that have derivatives through order $s$ that are elements of $L^p_\mu(Q)$. The Sobolev spaces for real smoothness parameter $s>0$ are obtained as interpolation spaces from those that have integer order. \cite{hangelbroek2012polyharmonic} We write $W^{s,p}(Q,Y)$ for the $Y$-valued Sobolev spaces, $W^{s,p}(Q,Y):=(W^{s,p}(Q))^n$ when $Y=\RR^n$.

\subsection{Reproducing Kernel Hilbert Spaces}
\label{sec:kernelspaces}
A number of constructions in reproducing kernel Hilbert spaces are used in this paper, and the reader is referred to standard sources like \cite{berlinet2011reproducing,saitoh2013reproducing} for a good introduction. The summary here is necessarily brief. Throughout the paper $\knl:X\times X\rightarrow \RR$ will be an admissible kernel that defines a native space $H$ of real-valued functions over $X$. We always assume that the kernel $\knl$ is continuous, symmetric,  and strictly positive definite in this paper, which means that for any finite set $\Xi_M:=\{\xi_{M,i} \ | \ 1\leq i\leq M\}$ of distinct points, the kernel matrix  $[\knl(\xi_{M,i},\xi_{M,j})]:=\KK(\Xi_M,\Xi_M)\in \RR^{M\times M}$ is a positive definite matrix. Kernels of this type include the exponential kernel, inverse multiquadric kernel, or Matern-Sobolev kernel, among others \cite{wendland2004scattered}. 
 The space $H$ that is generated by a subset $\Omega \subseteq X$ is defined as 
$
H:=\overline{\text{span}\{\knl_x \ | \ x\in \Omega\}}
$ 
where the closure is taken with respect to the inner product that is defined as $\iH{\knl_x}{\knl_y}=\knl(x,y)$ for all $x,y\in X$. 
The RKHS $H$ is characterized by the fact that it satisfies the reproducing property: we have $\iH{\knl_x}{g}=g(x)$ for all $x\in X$ and $g\in H$. 

In this paper, samples are generated along the trajectories of a discrete dynamical system over a subset $Q\subset X$. The system is defined in terms of an unknown function $f$ as in Equation \ref{eqn:dynamics1}, and the function $f$ defines an associated RKHS that is important for our analysis.  For an RKHS $H$ and a mapping $f:Q\subset X\rightarrow Q\subset X$, we define the pullback space $f^*(H)$ to be the set $\{h:Q\rightarrow \RR\ | \ h=g\circ f, \text{ for some } g\in H\}$. The pullback space $f^*(H)$ is itself an RKHS, with its kernel $\knl^*:Q\times Q\rightarrow \RR$ defined as $\knl^*(q_1,q_2)=\knl(f(q_1),f(q_2))$ for all $q_1,q_2\in Q$. \cite{saitoh2013reproducing} The norm on the pullback space that is induced by the pullback kernel $\knl^*$ is equivalent to $\|h\|_{f^*(H)}:=\inf\{ \|g\|_{H} \ | \ h=g\circ f, g\in H\}.$ For a fixed $f:Q\rightarrow Q\subset X$ the Koopman operator $U_f$ is defined as $U_fg=g\circ f$: this definition makes sense for scalar-valued functions $g\in H$ as well as vector-valued functions $G\in H^n$. We define the trace or restriction operator $Tg:=g|_Q$. The space of restrictions $\RH:=T(H)$ is itself an RKHS space. It is induced by the kernel $\rknl:Q\times Q\rightarrow \RR$ that is the restriction of the kernel $\knl$, so that $\rknl(x,y):=\knl(x,y)$ for all $x,y\in Q$. \cite{berlinet2011reproducing} 
%\subsection{Interpolation and Projection}

In this section, we review a few well-known observations regarding interpolation and projection for spaces defined in terms of scattered kernel bases. We summarize these properties for the finite-dimensional space $H_M:=\text{span}\{\knl_{\xi_{M,i}}\ | \ \xi_{M,i} \in \Xi_M, 1\leq i\leq M\}$ that is defined in terms of scattered bases in $H$ that are located at the centers $\Xi_M:=\{\xi_{M,1}, \ldots, \xi_{M,M}\}$. We let $\Pi_M: H\rightarrow H_M$ be the $H$-orthogonal projection onto the closed subspace $H_M\subset H$. Since we have assumed that the kernel $\knl$ on $H$ is strictly positive definite, we have the representation

$$
\Pi_M g=\sum_{i,j=1}^M \left ( \KK(\Xi_M,\Xi_M) \right)^{-1}_{i,j} g(\xi_{M,j}) \knl_{\xi_{M,i}}.
$$

\noindent The projection operator $\Pi_M$ induces the $H$-orthogonal decomposition $H=H_M \oplus Z_M$ where $Z_M$ is the closed subspace $Z_M:=\{ f\in H \ | \ f(\xi)=0 \text{ for all } \xi \in \Xi_M \}.$
This decomposition implies that the complementary projection $I-\Pi_M$ satisfies  $((I-\Pi_M)g)(\xi)=0$ for all $\xi\in \Xi_{M}$. In other words, the orthogonal projection $\Pi_M g$ interpolates the function $g$ at the points in $\Xi_M$. When studying the convergence properties of estimates over the set $Q\subset X$, it is common that approximation errors are described in terms of the fill distance $h_{\Xi_M,Q}$ of the finite samples $\Xi_M\subset Q$ in the set $Q$. The fill distance is defined as 
$
h_{\Xi_M,Q}=\sup_{q\in Q} \min_{\xi\in \Xi_M} d_Q(q,\xi)
$ 
with $d_Q$ the metric on $Q$. 

The relationships among the $H$-orthogonal projection $\Pi_M$, the scattered bases $\knl_{\xi_{M,i}}$ in $H_M$, and the coordinate representation of $\Pi_M$, have been stated by considering $H_M$, a subset of the RKHS $H$ induced by the kernel $\knl:X\times X\rightarrow \RR$ of functions over the ambient space $X$. However, entirely analogous statements can be made for the projections onto finite-dimensional spaces $\RH_M$ and $f^*(H_M)$, which are defined simply by defining scattered bases in terms of the translates of the kernels $\rknl_{\xi_{M,i}}$ and $\knl^*_{\xi_{M,i}}$, respectively, for $\xi_{M,i}\in \Xi_M$. To simplify notation in this paper, we use the generic operator $\Pi_M$ to refer to any of the (generally different) projections from $H,\RH,$ or $f^*(H)$ onto $H_M, \RH_M$, or $f^*(H_M)$, respectively. The specific choice of the domain and range of $\Pi_M$ is clear from the context in any particular equation that follows.

\subsection{Koopman Theory and EDMD Algorithm}
\label{sec:edmd}

We have noted earlier that the Koopman operator $U_f$ is defined from the identity $U_f g:=g\circ f$ for some fixed function $f$. 
One of the most common ways of approximating $U_f$ is based on the  extended dynamic mode decomposition (EDMD) algorithm.  It has been studied carefully in recent work like \cite{klus2016,korda2018convergence}.  Here we summarize the EDMD method using nomenclature that is common in these references so that we can distinguish and compare it from the data-driven approximation studied in this paper.  
From a set of input-output samples, typically denoted $\{(x_i,y_i\}_{i=1}^
M$, of Equation \ref{eqn:dynamics1}, we define two data matrices $\bm{X}$ and $\bm{X}^+$ defined as follows,
\begin{align*}
    \bm{X} &= [x_1, \cdots, x_{M}]\in \RR^{d\times M}, \\
 \bm{X}^+ &= [x_2, \cdots, x_{M+1}]:=f(\bm{X})\in \RR^{d\times M}.
\end{align*}
We want to use these data matrices to construct estimates of the Koopman operator $U_f$. Since the operator $U_f$ is infinite dimensional, we define a finite dimensional subspace for building approximations. Suppose we choose basis functions $\{\psi_i\ | \ 1\leq i \leq N\}$ for our approximations. 
We introduce   the vector $\bm{\psi} = [\psi_1, \dots, \psi_N]^T$. 
With the definition of the basis functions in $\bm{\psi}$, data matrices for the  EDMD algorithm are then determined as follows
\begin{align*}
    \Psi(\bm{X}) &= [\bm{\psi}(x_1), \cdots,\bm{\psi}( x_{M})] \in \RR^{N\times M}, \\
    \Psi(\bm{X}^+) &= [\bm{\psi}(x_2), \cdots, \bm{\psi}(x_{M+1})]\in \RR^{N\times M}.
\end{align*}
Approximations of the Koopman operator $U_f$ are constructed in terms of the matrix solution $A_{N,M}\in \RR^{N\times N}$ of the minimization problem
\begin{align*}
A_{N,M}&:=\underset{A\in \RR^{N\times N}}{\text{arg min}} \left \| A\Psi(\bm{X}) - \Psi(\bm{X}^+)\right \|_F^2,
%\\
%&=\underset{A\in \RR^{N\times N}}{\text{arg min}} \sum_{i=1}^M\left \|A\bm{\psi}(x_i)-\bm{\psi}(y_i)
%\right \|^2_2.
\end{align*}
with $\| \cdot \|_F$ the Frobenius norm on matrices. The solution of this minimization problem is given by $$A_{N,M}=\Psi(\bm{X}^+)
\left (  \Psi^T(\bm{X}) \Psi(\bm{X})\right )^{-1}\Psi^T(\bm{X}).$$ Finally, the approximation $U_{N,M}$ of the Koopman operator $U_f$ is given by 
\begin{align}
\left (U_{N,M} g \right )(\cdot):= \alpha^T A_{N,M} \bm{\psi}(\cdot) \label{eq:edmd_def}
\end{align}
for any function $g(\cdot):=\alpha^T \bm{\psi}(\cdot)$
with $\alpha\in \RR^N$.

\section{Problem Formulation}
The analysis in this paper exploits the relationship between approximation of Koopman operators and certain problems of statistical or machine learning. The distribution-free learning problem is a general problem that has been studied in a wide variety of contexts, most commonly for mappings between Euclidean spaces. \cite{gyorfi2006distribution}. As in \cite{powell2020learning} we are interested in this paper in formulating and solving this problem, not over some known compact subset such as $\Omega:=[a,b]^d\subset \RR^d$, but rather over the unknown manifold $Q$. 

There are several features of this problem that make it challenging in comparison to the classical problem over $\RR^d$. It is first necessary to state the learning problem precisely in terms of operators and functions defined over manifolds, but this seems to follow the form of the Euclidean case fairly closely. One complicating factor for the case at hand, however, is that samples we study in this paper are not generated from some independent and identically distributed (IID) measurement process. Rather, samples are collected along the trajectory of a discrete, typically nonlinear dynamic system over the unknown manifold, and the evolution law defining the samples is also unknown. Even if it were known, the fact that it defines a dependent process itself is a significant complication over the most common conventional case.

An additional, troublesome issue that this paper addresses is the question of how to define the function spaces defined over the unknown manifold when building approximations of the solution to the distribution-free learning problem. We want to build approximations in such a way that we obtain realizable algorithms that yield rates of convergence that can be proven in a spirit similar to the strategy in Euclidean spaces. This will mean that the regularity of $Q$, as a type of embedded manifold, will play a role in selecting the spaces for the formulation and its approximation.

\subsection{Distribution-Free Learning in Euclidean Spaces}
\label{sec:regressionRn}
We begin by summarizing the general structure of the distribution-free learning problem for regression over compact subsets of $X:=\RR^d$, and subsequently, we discuss how the problem is cast in terms of the unknown manifold $Q$ for certain types of discrete evolution laws. When the distribution-free learning problem is applied to regression, we are faced with building approximations of some unknown function over $X$ that takes values in $Y=\RR^n$. The typical case considers samples $\{(x_i,y_i)\}_{i=1}^M\subset X \times Y$. Here $X$ is the space of inputs and $Y$ is the output space. We can view the samples as perhaps noisy measurements of the functional relationship $y=G(x)$ where $G:X\rightarrow Y$ is unknown. As noted previously, it is usually assumed that the samples are independent and identically distributed, with $\nu$ the probability distribution over $X\times Y$ that generates the samples.

In the distribution-free learning problem, the probability distribution $\nu$ is unknown. The aim of learning theory is to build an approximation of the unknown function using the samples that makes some error measure small. Here we use the familiar quadratic error for arbitrary functions $g:X\rightarrow Y$, 
\begin{equation}
\label{eq:Emu}
E_\mu(g):=\int_X \int_Y \|y-g(x)\|_{Y}^2 \nu(dx,dy),
\end{equation}
where again we interpret $y=G(x)$ in the ideal, noise-free case. A great deal is known about the structure of this particular problem in the space $L^2_\mu(\RR^d,Y)$, or even in $L_\mu^2(\Omega,Y)$ with $\Omega$ a nice compact subset of $X$ like $\Omega=[a,b]^d$. It is known that the minimization of the quadratic error $E_\mu$ above is equivalent to minimizing the expression
\begin{equation}
\label{eq:Emu1}
E_\mu(g) =\int_X  \|g(x)-G_\mu(x)\|_{Y}^2 \mu(x) + E_\mu(G_\mu)
\end{equation}
where the measure $\nu$ is written as $\nu(dx,dy)=\mu_x(dy)\mu(dx)$, $\mu_x(dy)$ is the conditional probability over $Y$ given $x\in X$, $\mu(dx)$ is the marginal measure of $\nu$ over $X$, and $G_\mu(x):=\int_Y y\mu_x(dy)$ is the regressor function. Note that the term $E_\mu(G_\mu)$ does not depend on $g$. Additionally, we see that the term on the right of Equation \ref{eq:Emu1} is minimized when $g(x) = G_\mu(x)$. Thus, the regressor $G_\mu$ is the optimal minimizer over $L^2_\mu(\RR^d,Y)$. But it cannot be computed from this closed form expression since $\nu$, and therefore $\mu_x$, is unknown. In practice, then, the ideal error measure above is replaced with the discrete error

\begin{align}
\label{eq:EM}
E_M(g):=\frac{1}{M} \sum_{i=1}^M  \|y_i-g(x_i)\|_Y^2
\end{align}

\noindent that depends on the samples $\{(x_i,y_i)\}_{i=1}^M$. The method of empirical risk minimization (ERM) seeks to minimize $E_M$ instead of $E_\mu$. (Note that the minimization in Equation \ref{eq:EM} is over sequentially collected samples. This is the conventional or traditional assumption in the formulation of ERM. When discussing ERM we always use this notation, which is a special case of the setup in this paper.)  Note that the discrete error $E_M$ above can be evaluated for any given $g$. When estimates 
$
G_{M,N}={\text{arg inf}}\ \{ E_M(g) \ | \ g\in H_N\}
$ 
that minimize the empirical risk are calculated for some $N$ dimensional space $H_N$ of approximants, it is then possible to address in what sense $G_{M,N}\rightarrow G_\mu$ as the number of samples $M$ and the dimension $N$ approach infinity. The theory for such approximations is mature, and a summary of the state art in these cases can be found \cite{devore2006approximation}.

\subsection{Distribution-Free Learning over an Unknown Manifold}
\label{sec:randomIC}
The manifold estimation problem studied in this paper modifies the classical learning problem above in a number of ways: some changes are minor, while others lead to subtle issues. 
 Note carefully that the samples generated by  Equations \ref{eqn:dynamics1} are defined from a  deterministic, nonlinear, dependent measurement process, not one that is IID. 
 It turns out that the assumption that samples are IID is central to many of the precise estimates of errors in distribution-free learning theory.  We will analyze the system in Equation \ref{eqn:dynamics1} in two ways in this  and the next few sections. In this section, we assume that the samples $\{(\xi_{M,i},\upsilon_{M,i})\}_{i=1}^
M$ are generated by choosing initial conditions randomly according to the unknown measure $\mu$ on $Q$, and then observations $\upsilon_{M,i}$ are defined as in Equation \ref{eqn:dynamics2}. That is, the samples $\{(\xi_{M,i},\upsilon_{M,i})\}_{i=1}^
M$ are understood as the initial condition responses of the system in Equations \ref{eqn:dynamics1} and \ref{eqn:dynamics2} as the initial condition is chosen randomly over the unknown manifold $Q$. (Later we extend our analysis to the case when samples are dependent and determined along a trajectory.) 
In this case, the measure $\nu$ that generates the samples has the particular form $\nu(dx,dy)=\delta_{G(f(x))}(dy)\mu(dx)$, and the regressor function is 
$
G_\mu(x):=\int_Y y \delta_{G(f(x))}(dy) =G\circ f
$.
With these definitions, the ideal error $E_\mu$ in Equations \ref{eq:Emu} and \ref{eq:Emu1} makes sense for the Lebesgue space of functions $L^2_\mu(Q,Y)$ defined over the manifold $Q$ without change. Likewise, the empirical error $E_M$ in Equation \ref{eq:EM} makes sense, since the evaluation operator is well-defined on $H\subset C(X)$.  
In this situation, in contrast to the form of the regression error functional in Equation \ref{eq:Emu1}, we consider the ideal error 
\begin{align*}
    E_\mu(g)&
    %=\int_Q \|(g\circ %f)(x) -( G\circ %f)(x) \mu(dx)  + %E_\mu(G\circ f),\\ 
    %&
    = \int_Q \|(U_f(g-G))(x)\|_Y^2 \mu(dx) + E_\mu(U_f G),
\end{align*}
where $U_f h=h\circ f$ defines the Koopman operator, $U_f$ induced by the unknown function $f$ that induces the unknown dynamics. Note that this expression can be interpreted as a measure of the error in estimating the regressor $H_\mu = G \circ f$ by $U_f G$. The corresponding discrete error function is then 
\begin{align}
E_M(g)&:=\frac{1}{M} \sum_{i=1}^M \|\upsilon_{M,i}-(g\circ f)(\xi_{M,i})\|_Y^2=\frac{1}{M}\sum_{\ell =1}^n\sum_{i=1}^M {E^{(\ell)}_i(g)} = \frac{1}{M}\sum_{\ell =1}^n{E_M^{(\ell)}(g^{(\ell)})} \notag \\
%& = \frac{1}{M} \sum_{i=1}^M \|(G\circ f)(x_i)-(g\circ f)(x_i)\|_Y^2, \label{eq:risk_pull} \\
&=\frac{1}{M}  \sum_{\ell =1}^n {\sum_{i=1}^M 
\left | (G^{(\ell)}\circ f)(\xi_{M,i})-(g^{(\ell)}\circ f)(x_i) \right |^2}, 
\label{eqn:empRisk}
\end{align}
where we have introduced  the component-wise notation $g:=(g^{(1)},\ldots,g^{(n)})$ and $G:=(G^{(1)},\ldots, G^{(n)})$. 
\subsection{Projection for  Regression in the Pullback Space}
\label{sec:pullback_projection}
We start with a straightforward analysis under the assumptions of the last Section \ref{sec:randomIC}. We  interpret the empirical risk minimization in terms of projection or interpolation in a pullback space. We have chosen to start with this elementary case, since it makes clear some of the issues that distinguish the learning problem in this paper from the usual one in Euclidean space. We consider the component-wise minimization 
$$
    g_M^{(\ell)}:= \text{argmin} \{ E^{(\ell)}_M(g) \ | \ {g \in H_M} \} 
$$
with $H_M:=\text{span}\{ \knl_{\xi_{M,i}}\ | \ \xi_{M,i}\in \Xi_M\}$ 
that is associated with Equation \ref{eqn:empRisk}. 
\begin{theorem}
\label{th:proj_pullback}
Suppose that $f(\Xi_M):=\{f(\xi_{M,i})\ | \ \xi_{M,i}\in \Xi_M\}$ are distinct points. 
The empirical risk minimizer $g_M^{(\ell)}$ satisfies  
\begin{align}
    (U_f g^{(\ell)}_M)(\cdot):= \sum_{i,j}^M 
    \left ( \KK(f(\Xi_M),f(\Xi_M)) \right )^{-1}_{i,j} \upsilon_{M,i}^{(\ell)} \knl(f(\xi_{M,i}),f(\cdot)) \label{eq:erisk1}
\end{align}
\end{theorem}
\begin{proof}
The proof of this theorem is a direct consequence of scattered data interpolation over the pullback space $f^*(H)$. Define the space of approximants in the pullback space $f^*(H)$ as
$$
H^*_M:=\text{span} \left \{ \knl^*_{\xi_{M,i}}\ | \ \xi_{M,i}\in \Xi_M\right \}\subset f^*(H). 
$$ 
We can consider another minimization problem over the pullback space $f^*(H)$, 
\begin{align}
\label{eq:hk1}
h^{(\ell)}_M = \underset{h\in H^*_M \subset f^*(H)}{\text{argmin}}\ {\frac{1}{M} \sum_{i=1}^M |(\upsilon_{M,i}^{(\ell)}-h)(\xi_{M,i})|^2}, 
\end{align}
But Equation \ref{eq:hk1} is solved by the function $h^{(\ell)}_M\in f^*(H)$ that interpolates the samples $\{(\xi_{M,i},\upsilon_{M,i}^{(\ell)})\}_{i=1}^{M}$, since in this case the right hand side of Equation \ref{eq:hk1} is precisely zero. The solution of this intepolation problem is given by $$h_M^{(\ell)}(\cdot) = \sum_{i,j} [\knl^*(\xi_{M,i},\xi_{M,j})]^{-1}\upsilon_{M,i}^{(\ell)} \knl^*(\xi_{M,i},\cdot) $$ where $\knl^*$ is the pullback kernel on the pullback space $f^*(H)$. The right hand side of Equation (4) is precisely the coordinate expression for the  $f^*(H)$-orthogonal projection onto $H^*_M:=f^*(H_M)$. Since the kernel $\knl$ is assumed to be positive definite over $X\times X$, it follows that the pullback kernel $\knl^*(\cdot,\cdot)=\knl(f(\cdot),f(\cdot))$ is also strictly positive definite over $Q\times Q$. The assumption that $f(\Xi)$ are distinct points ensures that the Grammian in Equation \ref{eq:erisk1} is invertible. From these observations it follows that $$g^{(\ell)}_M(\cdot):= \sum_{i,j}^M 
    \left ( \KK(f(\Xi_M),f(\Xi_M)) \right )^{-1}_{i,j} \upsilon_{M,i}^{(\ell)} \knl(f(\xi_{M,i}),\cdot)$$  is a minimizer of $E^{(\ell)}_M$ over $H_M$.
\end{proof}

Overall, Theorem \ref{th:proj_pullback} is quite informative. It shows that the distribution-free learning problem at hand, which is expressed in terms of an unknown manifold $Q$, measure $\mu$ on $Q$, and function $f:Q\rightarrow Q$, can in principle be solved as a simple regression problem on the pullback space $f^*(H)$. Unfortunately, the solution in Equation \ref{eq:erisk1} is not of practical use in the problem at hand, since the solution depends on the unknown function $f$. This makes clear a point that distinguishes the learning problem on manifolds that we study here, where the dynamics propagates according to some unknown function $f$, from the classical regression problem. 

\subsection{Data-Driven Approximations} 
\label{sec:data_driven}
Finally, we study the uncertain system in Equations \ref{eqn:dynamics1} and \ref{eqn:dynamics2} in its full generality. The samples $\{(\xi_{M,i},\upsilon_{M,i})\}_{i=1}^M$ are collected along the deterministic, nonlinear system at discrete times $\TT_M:=\{k_i\ | \ 1\leq i\leq M\}$. Now the measure $\mu$ is taken to be some convenient probability measure over $Q$ for defining average error, but it does not generate the samples randomly. 
In this section we define a data-driven approximation $U^M_f$ of the Koopman operator, which is used to bound the excess risk, which is defined by the left hand side of the inequality
\begin{align*}
   &E_\mu(g)-E_\mu(U_fG)= \int_Q \|(U_f(G-g))(x)\|_Y^2 \mu(dx), \\
   %&\hspace*{.25in}  \leq \int_Q \sup_{q\in Q} \|(U_f(G-g))(Q)\|_Y^2 \mu(dx),
    %=\|U_f(G-g)\|^2_{C(Q,Y)}
   % \\
    &\hspace*{.25in} \lesssim \left (
    \max_{1\leq \ell \leq n} \|U_f(G^{(\ell)})-U_f(g^{(\ell)})\|_{C(Q,\RR)}
    \right )^2.
\end{align*}
We make the structural assumption that the pullback space $f^*(H)$ of functions over the manifold $Q$ is continuously embedded in the space of restrictions $\RH:=T(H)$, $f^*(H)\hookrightarrow \RH$, and furthermore, that the space of restrictions $\RR$  is embedded continuously in the Sobolev space $W^{s,2}(Q)$.  
Following the strategy in \cite{paruchuri2020koopman}, we define the data dependent approximation $U^M_fg$ of $U_fg$ to be given by 
\begin{align*}
    U^M_fg:=\Pi_M( (\Pi_M g)\circ f),
\end{align*}
where this equation is interpreted componentwise for the vector-valued function $g$. That is, 
$
U^M_f g:=( U^M_f g^{(1)}, \ldots , U^M_f g^{(n)} )
$
for $g=(g^{(1)},\ldots, g^{(n)})$. 
Based on the discussion in Section \ref{sec:kernelspaces}, the coordinate expression for each component of this operator is 
\begin{align}
    &U^M_fg^{(\ell)}:=\sum_{i,j=1}^M
    \left ( \KK(\Xi_M,\Xi_M)\right)^{-1}_{i,j} \notag \\
    &
    \times \sum_{p,q=1}^M \left ( \KK(\Xi_M,\Xi_M)\right )_{p,q}^{-1} g^{(\ell)}(\xi_{M,q}) \KK(\Xi_M,f(\Xi_M))_{p,j} \knl_{\xi_{M,i}}. \label{eq:UMf_coord}
\end{align}
\begin{theorem}
\label{th:approx_koopman} 
Suppose that $f^*(H)\hookrightarrow R \approx W^{s,2}(Q)$. Then there are constants $C,H>0$ such that for all $\Xi_M\subset Q$ that satisfy $h_{\Xi_M,Q}\leq H$, we have 
$$
\|(U_fG)(x)-(U_f^M G)(x)\|_Y \leq C h^{s-m}_{\Xi_M,Q} \|G\|_{W^{s,2}(Q)}
$$
for all $G\in W^{s,2}(Q)$ and $x\in Q$. 
\end{theorem}
\begin{proof}
We define the approximation $U_f^M G$ as above, $U_f^MG=\Pi_M( (\Pi_M G)\circ f)$. Note that the coordinate expression for $U_f^MG$ can be calculated using the observations $\{(\xi_{M,i},\upsilon_{M,i})\}_{i=1}^M=\{(\xi_{M,i},G(f(\xi_{M,i}))\}_{i=1}^M$. (Here we assume that we have collected the one additional measurement $\xi_{M,M+1}=f(\xi_{M,M})$ that is also needed in the coordinate expression.) The proof of this theorem follows that strategy in \cite{paruchuri2020koopman}, generalizing it to the case when $Y=\RR^n$, where 
\begin{align*}
   &  \|(U_fG)(x)-(U^M_f G)(x)\|_Y \\
    & \hspace*{.1in} \lesssim \max_{1\leq \ell \leq n}|(G^{(\ell)}\circ f)(x)-(\Pi_M((\Pi_M G^{(\ell)})\circ f)(x)|, \\
    &\hspace*{.1in} \leq \max_{1\leq \ell \leq n}
    \biggl \{ 
    \underbrace{\|(I-\Pi_M)G^{(\ell)}\|_{C(Q,\RR)}}_{\text{term 1}} \\
     & \hspace*{.5in} + \underbrace{\|(I-\Pi_M)((\Pi_MG^{(\ell)})\circ f)\|_{C(Q,\RR)}}_{\text{term 2}}
    \biggr \}.
\end{align*}
Both term 1 and term 2 above are bounded using the many zeros Theorem \ref{th:manyz} on the manifold $Q$. Term 1 is bounded using the fact that $\RH \approx W^{s,2}(Q)\hookrightarrow W^{m,2}(Q)$ and subsequently  
\begin{align}
    \|(I-\Pi_M)G^{(\ell)}\|_{C(Q,\RR)} & \lesssim \|(I-\Pi_M)G^{(\ell)}\|_{W^{m,2}(Q)}, \notag \\ 
    & \lesssim  h^{s-m}_{\Xi_M,Q}  \|(I-\Pi_M)G^{(\ell)}\|_{W^{s,2}(Q)},\notag \\
    & \lesssim h^{s-m}_{\Xi_M,Q} \|G^{(\ell)}\|_{W^{s,2}(Q)}. \notag
\end{align}
In the above we have applied Theorem \ref{th:manyz} to the term $(I-\Pi_M)G^{(\ell)}$, which vanishes on $\Xi_M$. We also use that fact that $\Pi_M$ is a bounded operator on $\RH=W^{s,2}(Q)$. We can similarly bound the second term:   $\text{(term 2)} \lesssim h^{s-m}_{\Xi_M,Q}\|U_f\Pi_M G^{(\ell)}\|_{W^{s,2}(Q)}$. By virtue of the assumption that $f^*(H)\hookrightarrow \RH\approx W^{s,2}(Q)$, we know $(\text{term 2}) \lesssim  h^{s-m}_{\Xi_M,Q} \|G^{(\ell)}\|_{W^{s,2}(Q)}$ also. Collecting the different inequalities for each of the components $1\leq \ell \leq n$ finishes the proof. 
\end{proof}

We close this section by showing that the analysis of this section enables similar error bounds to be defined for certain cases of the EDMD algorithm.

\begin{theorem}
\label{th:edmd}
Suppose that $U_{M,M}$ is the approximation of the Koopman operator derived by the EDMD algorithm in Equation \ref{eq:edmd_def} using the samples
$\{(\xi_{M,i},\upsilon_{M,i})\}_{i=1}^M$,  so that  the basis vector $\bm{\psi}$ is selected as $\bm{\psi}:=\{ \knl_{\xi_{M,1}},\ldots, \knl_{\xi_{M,M}} \}^T$.  Under the hypotheses of Theorem \ref{th:approx_koopman}, we have 
$$
\|(U_fG)(x)-(U_{M,M} G)(x)\|_Y \leq C h^{s-m}_{\Xi_M,Q} \|g\|_{W^{s,2}(Q)}
$$
for the EDMD approximation $U_{M,M}$ of the Koopman operator $U_f$. 
\end{theorem}
\begin{proof}
In general, the EDMD approximation of the Koopman operator depends on two parameters, the number of samples $M$ and the number of basis functions $N$. It is a more general method of building an approximation of $U_f$ than $U_f^M$ in this paper. However, when $M=N$ and the basis vector $\bm{\psi}$ is chosen as $\bm{\psi}:=\{ \knl_{\xi_1},\ldots, \knl_{\xi_M} \}^T$, then the coordinate expression in Equation \ref{eq:UMf_coord} for $U_f^M$ is precisely the same as that in Equation \ref{eq:edmd_def}. Since $U^M_f=U_{M,M}$, the error bound follows. 
\end{proof}
We close this section with the description of the  algorithm that generates an estimate that accords with the  error bounds  described above. Note that the discrete times $k_i \in \TT_M$ must be updated and stored in order to determine the selected outputs $\{y_{k_i}\}_{i=1}^M$ used in the estimate calculation. This list is updated adaptively based on two criteria: maintaining some minimal separation between points and generating samples that effectively cover the underlying manifold or subset of interest.  Minimal separation is controlled by the parameter $\eta$, which is needed to ensure numerical stability of the kernel matrix $K(f(\Xi_M),f(\Xi_M))$. It is well-known that a lower bound on the smallest eigenvalue of the kernel matrix can often be related to the minimal separation of the points under consideration. \cite{wendland2004scattered} Our numerical experiments have established that this is a critical feature of the estimation problem, one that cannot be overlooked.  Because of page limitations, we leave reporting of this qualitative aspect of the conditioning to a full, journal version of this paper to follow.  This algorithm is based on the assumption that the sequential input/output samples $\{(x_i,y_i)\}_{i\in\NN}$  generated by the governing Equations \ref{eqn:dynamics1} and \ref{eqn:dynamics2}, and measured by the agent, are available on a time scale that is much finer than that used for the generation of approximations. Estimation and approximation is carried out on a coarser time scale that is based on subselecting centers and outputs that satisfy separation and coverage criteria in the algorithm. 
\begin{algorithm}
\caption{\centering Data-Dependent Approximation of the Koopman Operator  Over an Orbit}
\label{alg: cent est}
\begin{algorithmic}
\STATE{\textbf{input}: Candidate sample library $\{x_i,y^{(\ell)}_i\}_{i=1}^m$, for component $\ell$, $1 \leq \ell \leq n$ of the output}
\STATE{\textbf{initialization}: $\Xi_M = \xi_{M,1} = x_1$, $M=1$}
\STATE{Subselect kernel centers}
\FOR{$x_i$ in $\{x_i\}_{i=2}^m$}
\STATE{if $\|\xi_{M,j}-x_i\| > \eta, \quad \forall \xi_{M,j} \in \Xi_M $}
\STATE{Add sample $x_i$ to the set of centers}
\STATE{$\Xi_{M+1} = \Xi_{M}\cup \{x_i\}$}
\STATE{Add sample index $i$ to the set of indexed times}
\STATE{$k_i = i$, $\TT_{M+1} = \TT_{M}\cup \{k_i\}$, $M = M+1$}
\ENDFOR
\STATE{Construct Kernel Matrix}
\STATE{$\KK(f(\Xi_M),f(\Xi_M)_{i,j} = [\mathfrak{K}(f(\xi_{M,i}),f(\xi_{M,j}))]$}
\STATE{Calculate coefficient $\{\alpha_j\}_{j=1}^M$ of estimate}
\STATE{$\alpha_i = \sum_{i,j}^M 
    \left ( \KK(f(\Xi_M),f(\Xi_M)) \right )^{-1}_{i,j} y^{(\ell)}_{k_j}$ }
\RETURN Estimate $U^M_F G^{(\ell)}(\cdot) = \sum_{i=1}^M \alpha_i \mathfrak{K}_{M,i}(\cdot)$ .
\end{algorithmic}
\end{algorithm}

\subsection{Sobolev Spaces by Restriction}
The approach in this paper relies on defining a reasonable kernel over the manifold $Q$ that is not known {\em a priori}. Before proceeding to consider specific exmaples and numerical studies, we summarize a few issues that must be considered in choosing a kernel. There are a wide range kernel functions, commonly referred to as radial basis functions (RBFs), the depend only on the distance between points on the underlying set. Many of these RBFs are defined over subsets of $\RR^d$ and depend on the Euclidean distance between points, see \cite{wendland2004scattered} for a summary. However, some kernels are defined over smooth manifolds using intrinsic information about the manifold, such as its intrinsic distance, its Riemannian metric, or its coordinate charts. One standard example to define the kernel in these cases uses eigenfunctions of the Laplace-Beltrami operator over the manifold and Mercer's Theorem to define the kernel.  We refer to these as intrinsic methods that define the kernel. Another intrinsic approach uses coordinate chars as in \cite{wendland2004scattered} Chapter 17. While the latter intrinsic methods to define the kernel on $Q$ would provide a path to approximations if $Q$ were known, they seem less well-suited to attack the problem when $Q$ itself is unknown.

In an extrinsic method a kernel over the embedded manifold is obtained by restricting the kernel defined on the ambient Euclidean space to the submanifold. As mentioned previously, we define a restricted kernel $\rknl : Q \times Q \to R$ from our embedded manifold $Q\subset \RR^d$, where
$\rknl(\cdot , \cdot) := \knl(\cdot , \cdot)|_{Q\times Q}.$
If $\knl$ is positive definite, so is its restriction $\rknl$ to $Q$, making $\rknl$ well-suited for approximation
problems over the submanifold. This philosophy is studied in detail in \cite{fuselier}. The main benefit of such an approach is that we do not necessarily need to know the submanifold $Q$, but rather just that our inputs to our kernel actually come from this restriction rather than the full ambient space. Another benefit of this approach is that there is a large collection of candidate kernels defined over Euclidean space from which to choose.

Even though extrinsic methods can be attractive, they do have their drawbacks. There are some well-known trade-offs from this approach primarily related to the loss of smoothness in approximating functions. We note that, for any given $k$-dimensional submanifold $Q \subset \RR^d$, the trace operator $\mathcal{T}_Q f = f|_Q$, which restricts functions to a $k$-dimensional manifold $Q$ actually extends to a continuous, and therefore bounded, linear operator $\mathcal{T}_Q: W^{s,2}(Q) \to W^{m,2}(Q)$ with $m \leq s-(d-k)/2$. Essentially, this means that restricting functions to the submanifold reduces smoothness by $1/2(d-k)$ for the reduction in dimension $(d-k)$  \cite{fuselier}. Note, that the rates of convergence of the estimates from Theorems \ref{th:approx_koopman} and \ref{th:edmd} are bounded by the fill distance raised to the power $s$ which depends on the smoothness of the space. This reduction in smoothness, therefore, also decreases the rate at which the estimate converges.

\section{Examples and Numerical Results}
This section summarizes numerical results for some examples of estimates given a family of observations $\{(\xi_{M,i},\upsilon_{M,i})\}_{i=1}^M$ collected over an unknown manifold, $Q$. We set the centers $\Xi_M:=\{\xi_{M,1},\ldots, \xi_{M.M}\}$ in this section. 
\subsection{Example: The Pendulum}
For our first example, we consider a simple case that can be expressed in closed form to illustrate the concepts behind the empirical estimate. The evolution corresponds to the dynamics of a pendulum with mass $m=1$ hanging from a massless rod of length $l=1$ that is approximated using the St\"ormer-Verlet discrete integration scheme as given in Example 1.4 in \cite{hairer2006geometric}. In this example, $x^{(1)}$ and $x^{(2)}$ correspond to the pendulum's momentum and position respectively. The equations of motion are given as 
\begin{align}
    \dot{x}^{(1)} &= -sin(x^{(2)}) = f_1(x^{(2)}) \\
    \dot{x}^{(2)} &= x^{(1)}
\end{align}

The discrete evolution is determined for the pendulum's position $x^{(2)}_i$, as well as approximations $x^{(1)}_i$ and $x^{(1)}_{i+1/2}$ of the derivative $x^{(1)}=\dot{x}^{(2)}$ at the $i^{th}$ and $i+1/2$ step. The approximations are calculated by
\begin{equation}
    x^{(1)}_i = \frac{x^{(2)}_{i+1}-x^{(2)}_{i-1}}{2h}
    \label{eqn:momentumApprox1}
\end{equation}
\begin{equation}
    x^{(1)}_{i+1/2} = \frac{x^{(2)}_{i+1}-x^{(2)}_{i}}{h}
    \label{eqn:momentumApprox2}
\end{equation}
where $h$ is the time step. Note that in this formulation $\ddot{x}^{(2)} = f_1(x^{(2)})$ which can be approximated by a second-order difference quotient
\begin{equation}
    x^{(2)}_{i+1}-2x^{(2)}_i+x^{(2)}_{i-1} = h^2f_1(x^{(1)}_i)
    \label{eqn:secondDiff}
\end{equation}
Using the positions $x^{(2)}_i$ and the momentum $x^{(1)}_i$ along with the second derivative $f_1(x^{(1)}_i)$ and Equations \ref{eqn:momentumApprox1}, \ref{eqn:momentumApprox2} and \ref{eqn:secondDiff}, we can determine a one step evolution $f: Q \to Q$ to get  $(x^{(1)}_{i+1},x^{(2)}_{i+1})$ defined by the following recursion
\begin{align}
    x^{(1)}_{i+1/2} &= x^{(1)}_i+\frac{h}{2}f_1(x^{(2)}_i)  \label{eqn:recursion1}\\
    x^{(2)}_{i+1} &= x^{(2)}_{i} +hx^{(1)}_{i+1/2}\\
    x^{(1)}_{i+1} &= x^{(1)}_{i+1/2}+ \frac{h}{2}f_1(x^{(2)}_{i+1}).
    \label{eqn:recursion3}
\end{align}
To generate the samples, the recursion is first taken over 256 steps, but the actual centers $\Xi_M$ were chosen at iterations $\{k_i\}_{i=1}^M$ such that there was sufficient separation between the centers. For this example, a total of $37$ centers were used. For a set of centers $\Xi_M$, the separation $$s_{\Xi_M} = \frac{1}{2} \min_{\xi_{M,i},\xi_{M,j}}d_Q(\xi_{M,i},\xi_{M,j})$$ corresponds to the largest possible radius of disjoint balls centered around the samples $\Xi_M$. As will be seen shortly, the separation will be key to numerical stability.

We build estimates for an unknown observable function $G:X:=\RR^2 \to \RR$ from samples collected over $Q  \subset \RR^2$. The function $G$, and therefore the outputs $y_i$, are defined for each $x_i := (x_i^{(1)}, x^{(2)}_i) \in \RR^2$ by 
\begin{equation}
    G(x_i)= y_i = 1.5-\text{sin}(x_i^{(2)})+\frac{1}{9}(x_i^{(1)})^2. \label{eqn:unknownFunc}
\end{equation}
 Figure \ref{fig:simulatedDataEstimate} depicts the estimates of the Koopman operator acting on  the  observable function $G$. In these studies, we use the Matern-Sobolev kernel, 
\begin{equation*}
    \knl^{MS}(\xi_{M,j},x) = \knl^{MS}_{\xi_{M,j}}(x) = \bigg(1+ \frac{\sqrt{3}\|x-\xi_{M,j}\|_2^2}{\beta }\bigg)\text{exp}{\bigg(-\frac{\sqrt{3}\|x-\
    \xi_{M,j}\|_2^2}{\beta }\bigg)}
\end{equation*}
with $\| \cdot \|_2$ the standard Euclidean norm over $\RR^2$, $\xi_{M,j}$ corresponds to the $jth$ kernel center, and the hyperparameter $\beta$ measures the decay of the kernel function away from the center. From the figure, it is evident that the surface, $U_f^M G$ generated from the RKHS basis intersects the true function surface $U_fG$ along the manifold $Q$ where the sample inputs are collected. This is a result of the interpolating properties of the RKHS approximation. 
\begin{figure}[htp]
    \centering
     \includegraphics[width = \textwidth]{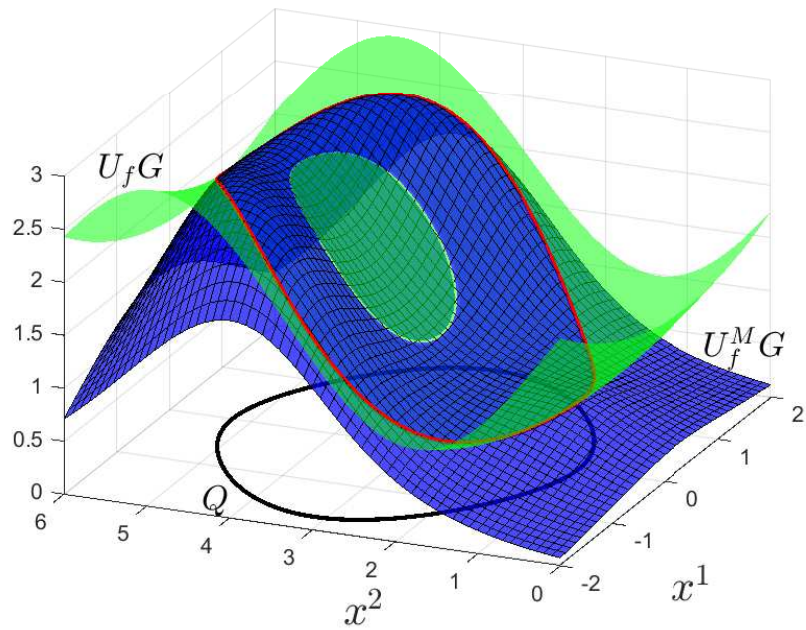}
    \caption{An example of the $U_f G$ approximation (in red) of $G$, represented by the light green surface.  Using $M = 37$ samples collected over the composition $U_f G$ defined by Equations \ref{eqn:recursion1} through \ref{eqn:recursion3}  and the unknown function \ref{eqn:unknownFunc}, the finite dimensional approximation $U_f^M G$ is generated from the RKHS represented by the mesh blue surface. Minimizing the empirical risk over the samples corresponds to interpolating the function over the samples as represented by the red dashed line.}
    \label{fig:simulatedDataEstimate}
\end{figure}

Approximations are also generated from compactly supported kernel functions as  defined in \cite{wendland2004scattered}. Some examples of these types of kernels include

\begin{align*}
    \knl^{CS}_1(\mathfrak{d})  = & (1-\mathfrak{d})_{+}^4(4\mathfrak{d}+1) \\
    \knl^{CS}_2(\mathfrak{d})  = &  (1-\mathfrak{d})_{+}^6(35\mathfrak{d}^2+18\mathfrak{d}+3) \\
     \knl^{CS}_3(\mathfrak{d}) =  & (1-\mathfrak{d})_{+}^8(32\mathfrak{d}^3+ 25\mathfrak{d}^2+8\mathfrak{d}+1) 
\end{align*}
where the radial distance $\mathfrak{d} = \|x-\xi_{M,j}\|_2$ is defined from the kernel center $\xi_{M,j}$, and the operator $(\cdot)_{+}$ returns the argument in the parenthesis if it is positive and zero otherwise. The positive-semidefinite property of the  operator $(\cdot)_{+}$ ensures each kernel is compactly supported. Figure \ref{fig:approx4DiffKernels} compares the estimate using the Matern-Sobolev kernel with $\beta = 1$ to the estimate using the compactly supported kernel $\knl^{CS}_2$. 

Roughly speaking, the performance of the EDMD algorithm in Figures \ref{fig:simulatedDataEstimate} and \ref{fig:approx4DiffKernels} can be assessed by considering two convergence regimes. Over the manifold $Q$, excellent convergence is attained as guaranteed in Theorem \ref{th:edmd}. For locations outside the manifold, we can interpret Figures \ref{fig:simulatedDataEstimate} and \ref{fig:approx4DiffKernels} as illustrating the accuracy of ``out of sample" predictions depending on the regularity of the unknown function $G$ and the hypothesis space. Out of samples predictions can have substantial error (that decreases as we approach the manifold). Notice how both estimates have substantial error away from the manifold with the estimate generated via compactly supported kernels quickly decaying to zero away from the manifold as a consequence of its compactly supported basis.

\begin{figure}[htp]
\begin{subfigure}{\textwidth}
 \centering
     \includegraphics[width = \textwidth,trim={0 0.1cm 0 1cm},clip]{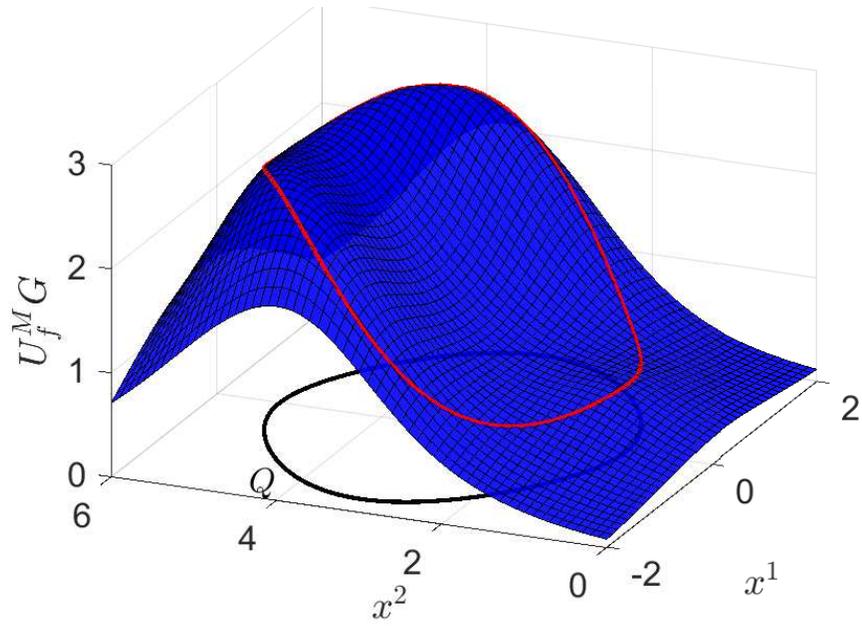}
    \caption*{(a)}
    \label{fig:kinematicEstimates}
\end{subfigure}
\begin{subfigure}{\textwidth}
\centering
\includegraphics[width = \textwidth,trim={0 0.1cm 0 1cm},clip]{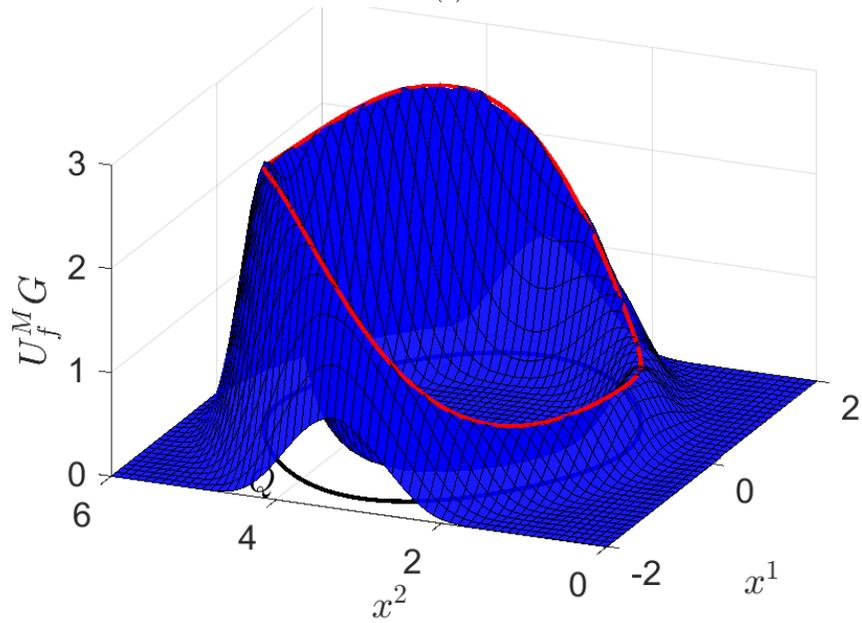}\caption*{ (b)}
\end{subfigure}
\caption{The estimate $U^M_f G$ of the composition $U_fG$ defined by Equations \ref{eqn:recursion1} through \ref{eqn:recursion3} and the unknown function defined by Equation \ref{eqn:unknownFunc}, using the Matern-Sobolev kernel (a) with $\beta = 1$ and the estimate using the compactly supported kernel $\knl^{CS}_2$ (b). Both estimates use the same number of samples $M=37$ and decay away from the manifold. However, the estimate (a) is still nonzero everywhere while the estimate (b) itself quickly decays to zero away from the manifold as a consequence of its bases being compactly supported.}
\label{fig:approx4DiffKernels}
\end{figure}

 Similar to many methods that use kernels for applications in learning theory \cite{williams2006gaussian}, kernel families like the Sobolev-Matern kernels can depend on additional constants, the hyperparameters, that affect the shape of the kernel bases. The hyperparameters can have a profound effect on the convergence of finite-dimensional approximations, depending on the application. As noted in \cite{powell2020learning} and \cite{powell2021distribution}, oscillations in the estimate can be evident depending on how we choose the hyperparameter, $\beta$, from the Matern-Sobolev kernel function. Here we illustrate the improvement that can be achieved using $\beta$. Figure \ref{fig:betaTesting} illustrates estimates generated by the Matern-Sobolev kernel for various $\beta$. In these cases, larger values of $\beta $ decrease the decay of the kernel centers and lead to a wider spread of each kernel. From Figure \ref{fig:betaTesting} it is clear that estimates generated from the larger $\beta$ are ``flatter" and exhibit less drastic changes in the estimate over small distances. As the $\beta$ parameter is decreased the decay rate increases from the center and lowers the spread of each kernel function. This can lead to the oscillations seen in the estimates for smaller $\beta$ values.

\begin{figure*}[htp]
\begin{subfigure}{0.49\textwidth}
\centering
\includegraphics[width = \textwidth,height = 6cm]{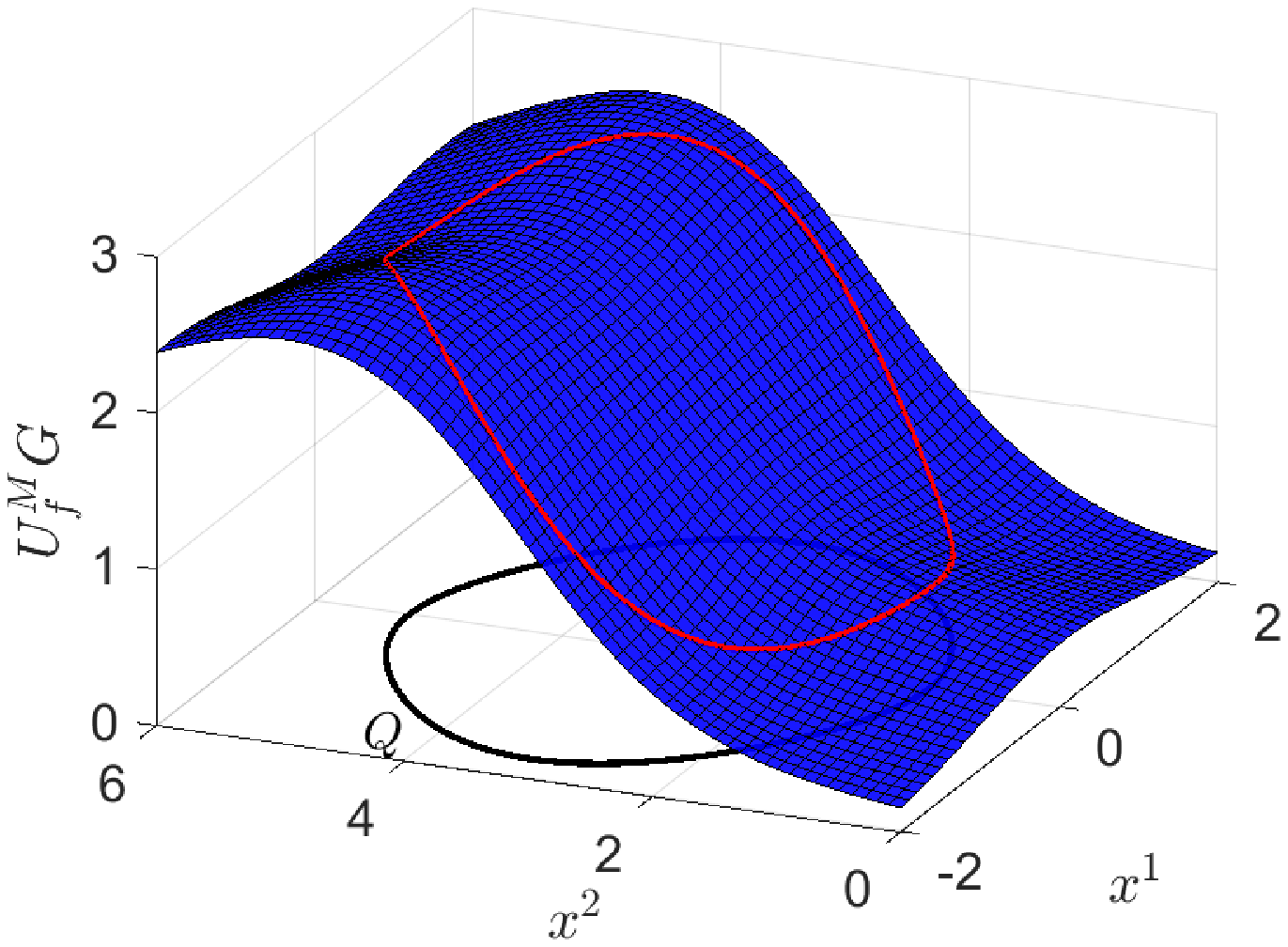}\caption*{ (a)}
\end{subfigure}
\begin{subfigure}{0.49\textwidth}
\centering
\includegraphics[width = \textwidth,height = 6cm]{SimDataEstimateBeta1.eps}\caption*{ (b)}
\end{subfigure}

\begin{subfigure}{0.49\textwidth}
\centering
\includegraphics[width = \textwidth,height = 6cm]{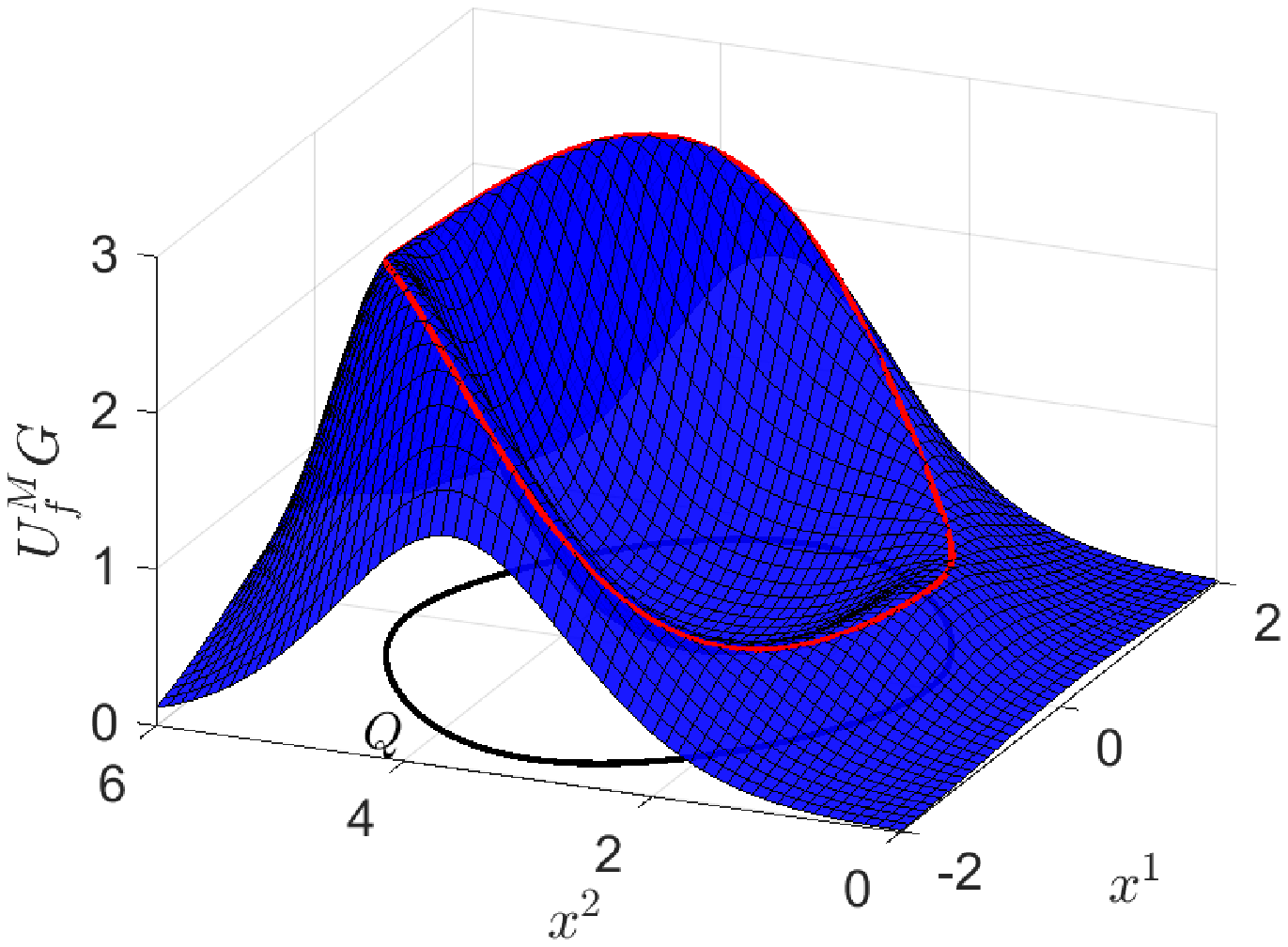}\caption*{ (c)}
\end{subfigure}
\begin{subfigure}{0.49\textwidth}
\centering
\includegraphics[width = \textwidth,height = 6cm]{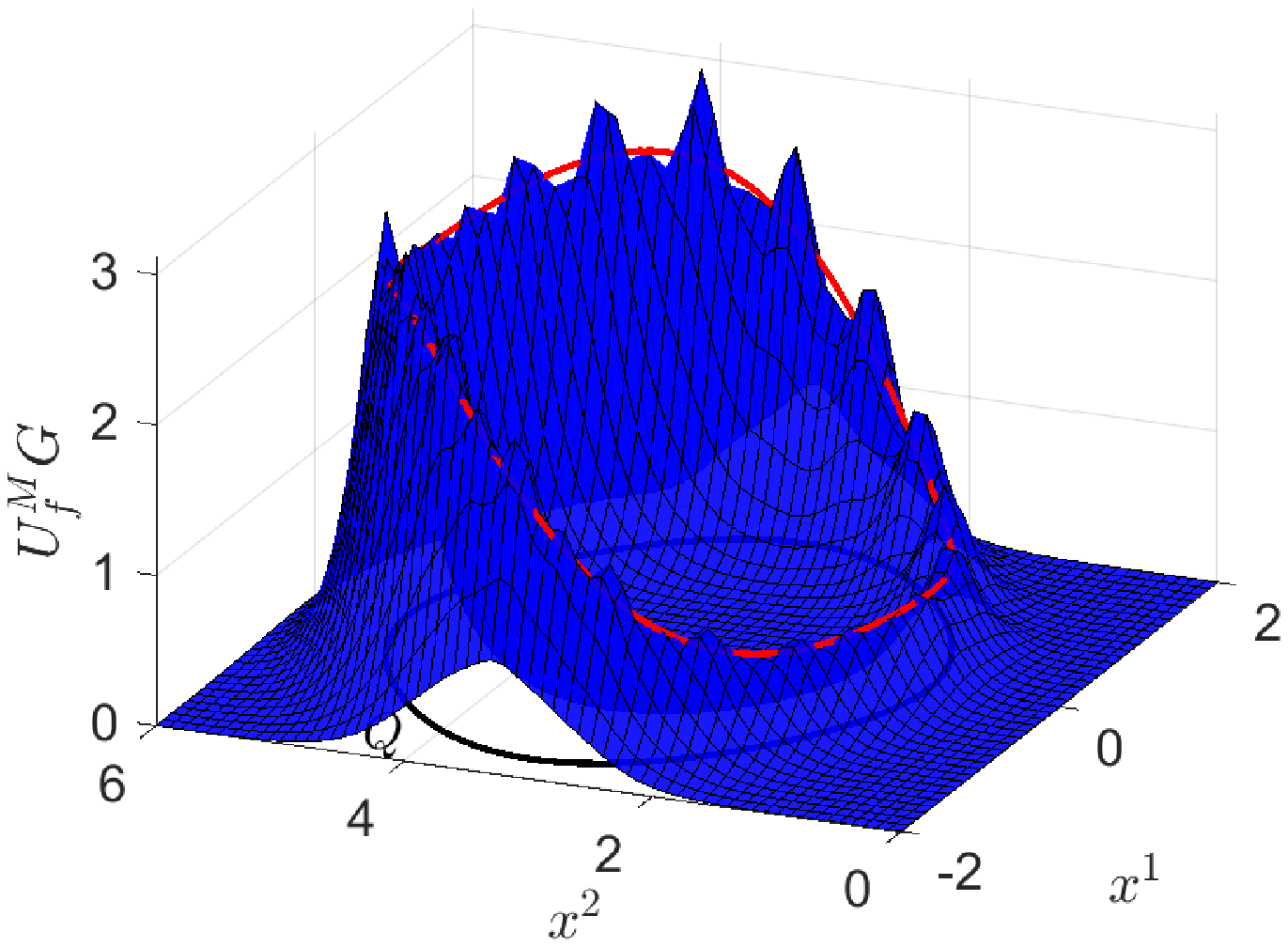}\caption*{ (d)}
\caption*{}
\end{subfigure}

\caption{The approximations $U^M_f G$ of the composition $U_fG$ defined by Equations \ref{eqn:recursion1} through \ref{eqn:recursion3} and the unknown function defined by Equation \ref{eqn:unknownFunc} for  various hyperparameters, $\beta$ of the Matern kernel: with $\beta= 5 $ (a), $\beta= 1 $ (b), $\beta= 0.5 $ (c),$\beta= 0.2$ (d). Each of the estimates is generated with the same number of centers $(M=37)$. Regardless of the hyperparameter selected, it is clear that each estimate is still determined by an algorithm which fits the estimate to the empirical samples collected along the red line over the manifold $Q$.}
\label{fig:betaTesting}
\end{figure*}

As illustrated in the comments above, Theorem \ref{th:edmd} can be an important tool in studying the fidelity of Koopman approximations over a manifold. As mentioned in the introduction and the theoretical results, the solution to the distribution free learning problem satisfies particular rates of convergence. Figure \ref{fig:errorEstimates} demonstrates the calculated error convergence of the estimate of the Koopman approximation to the true function as the number of samples increased. An increase in the number of centers collected on the manifold reduces the fill distance that bounds the error of the approximation. The decay rate is determined by selecting $t$ and $m$ as described in \cite{paruchuri2020koopman}. From the figure, it is clear that the error is bounded by $\mathcal{O}(h_{\Xi_M}^{t-m})$, which is represented by the dashed line above the curve. 

\begin{figure}[htp]
    \centering
     \includegraphics[width = \textwidth,trim={0 0.1cm 0 1cm},clip]{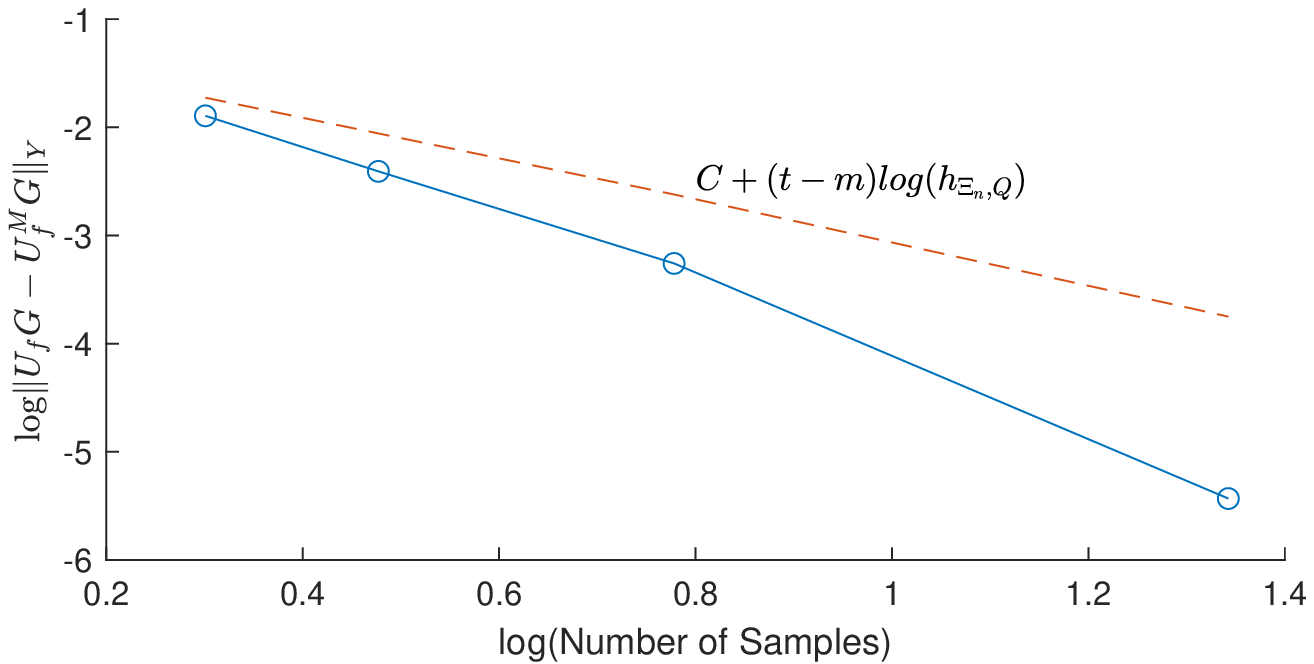}
    \caption{The rate of convergence of the estimate of the Koopman approximation. The rate of convergence is bounded by the fill distance represented by the dashed line above the curve. }
    \label{fig:errorEstimates}
\end{figure}

 Clearly, the accuracy in Theorem \ref{th:edmd} improves as the dimension of the space of approximation increase. However, practical use of the EDMD algorithm must also consider numerical stability as the dimension of function space increases. While stability of numerical approximations is well-documented in the literature on kernel approximations \cite{wendland2004scattered} it is not discussed anywhere in the recent literature on approximations of Koopman operators \cite{mezic2020koopman,klus2020}. For completeness, we review how numerical stability of Koopman approximations can be studied using results from kernel interpolation and approximation.

The solution, $U_f^M G$, of Equation \ref{eq:erisk1} comes from projecting the true solution, $G$, onto a subspace, $H_M = \text{span}\{\knl_{\xi_{M,i}}| \xi_{M,i} \in \Xi_M\}$, defined over the set of centers, $\Xi_M$. This solution satisfies the equation
\begin{equation}
    \bm{K}_M \bm{\alpha}^* = \bm{G}(\Xi_M)
    \label{eqn:exactCoefs}
\end{equation}
 where we define the kernel matrix $\bm{K}_M:=\KK(\Xi_M,\Xi_M)$, output vector $\bm{G}(\Xi_M):=[G(\xi_{M,1}),\dots ,G(\xi_{M,M})]^T$ and $\bm{\alpha}^* \in \RR^M$ is a vector of coefficients corresponding to the $H_M$ projection. However, noise in both the sample inputs and the output measurements as well as computer precision effectively results in perturbations, $\Delta \bm{K}_M$ and $\Delta \bm{G}$ of the kernel matrix $\bm{K}_M$ and the output vector $\bm{G}(\Xi_M)$ respectively. In practice, a numerical solution $\hat{\bm{\alpha}}$ is the exact solution of the nearby equation 
\begin{equation}
    (\bm{K}_M + \Delta \bm{K}_M) \hat{\bm{\alpha}} = \bm{G}(\Xi_M) +\Delta\bm{G}.
    \label{eqn:numericalCoefs}
\end{equation}
The bounds between the exact solution and the numerical calculation are known to satisfy
\begin{equation}
    \| \bm{\alpha}^* - \hat{\bm{\alpha}} \| \lesssim C \text{cond}(\bm{K}_M)\|\Delta \bm{K}_M\| \| \Delta \bm{G}\|.
    \label{eqn:stableBounds}
\end{equation}
where the condition number $\text{cond}(\bm{K}_M) = \|\bm{K}_M\| \|\bm{K}^{-1}_M\|$. Using norm equivalences, if the uniform norm is substituted into Equation \ref{eqn:stableBounds}, the logarithm of the condition number can be used to estimate the number of digits of accuracy lost.

We investigated the numerical stability of the EDMD estimates by examining the condition number of the kernel matrix from sets of kernel centers which differed by some predetermined quasi-uniform spacing.  Sample inputs were first generated by 256 steps of the recursion defined by Equations \ref{eqn:recursion1} through \ref{eqn:recursion3}. The centers used to generate the estimates were then selected to satisfy a particular fill distance. The number of centers used, therefore, ranged from 2 to 64 with decreases in the specified fill distance leading to an increase in the number of centers. Figure \ref{fig:uniformSeparationSmoothnessTest} demonstrates the condition number of kernel matrices using various compactly supported kernels as the minimum spacing between the kernel centers is decreased. The slopes of each curve demonstrate that for increasing order of kernels, which belong to more restricted smooth spaces, the computation becomes more unstable compared to the kernels which belong to lower order smoothness spaces. 
Similarly, Figure \ref{fig:uniformSeparationBetaTest} examines the condition number given the minimum spacing for kernel matrices defined by the Matern-Sobolev kernels with various hyperparameters, $\beta$. As mentioned previously, larger $\beta$ leads to a slower decay of the kernel function and, therefore, a larger  ``spread'' of the kernel about its center. From the figure it is evident that this larger spread increases the condition number of the corresponding kernel matrices. The numerical instability correspondingly suffers as is evident from the higher condition number for curves of increasing $\beta$. 
 
 Overall, in terms of our comments following Equation \ref{eqn:stableBounds} above, the range of the condition numbers in Figures \ref{fig:uniformSeparationSmoothnessTest} and \ref{fig:uniformSeparationBetaTest} is significant. Even with double precision calculations, many significant figures can be lost for these condition numbers. Thus, while Theorem \ref{th:edmd} provides a strong upper bound on accuracy, the upper bounds may not be realizable in practical computations. This topic has been the focus of much research in scattered data interpolation and approximation, but has not been studied in the EDMD literature.

On the other hand, we may be interested in the effect of non-uniform sampling. In Figures \ref{fig:uniformSeparationSmoothnessTest} and \ref{fig:uniformSeparationBetaTest}, the spacing between centers is relatively uniform amongst the samples and, therefore, the separation decreases along with the fill distance making it difficult to determine effects based solely on the separation. While the fill distance is the critical feature in determining the theoretical bounds on the error, the next set of results emphasized the numerical stability is dependant on the effects of the separation alone. As mentioned in \cite{wendland2004scattered}, lower bounds for $\lambda_{\text{Min}}(\bm{K}_M)$ can be solely determined as functions of the separation distance $s_{\Xi_M}$. Strictly speaking, the condition number depends on both the minimum and maximum eigenvalues of the kernel matrix. However, \cite{wendland2004scattered} notes that the condition number seems to depend most strongly on $\lambda_{\text{Min}}(\bm{K}_M)$. Specifically, the lower bound for the minimum eigenvalue of a kernel matrix defined by compactly supported kernels of $k$ smoothness as defined in \cite{wendland2004scattered} is a function $s_{\Xi_M}^{2k+1}$ of the separation distance. To investigate the effects of this decreasing bound, samples are generated following the recursion of Equations \ref{eqn:recursion1} through \ref{eqn:recursion3} for 256 iterations. Of those samples, centers are selected satisfying a particular fill distance along the manifold. An additional center $\xi_{M+1}$ is then generated at some smaller distance from the initial center $\xi_{M,1}$. Consequently, the fill distance in these examples remains constant for each curve while decreasing the separation alone. We then examine the minimum eigenvalue of the kernel matrix as a function of the spacing between two samples.

Figure \ref{fig:nonuniformSeparation} illustrates the effect of decreasing the separation distance between just two centers while keeping the fill distance constant. Each curve differs by the initially set fill distance of the set of centers. Decreases in the fill distance inadvertently result in decreases in the separation as demonstrated by the shift in decreasing direction for the minimum eigenvalue between curves. However, it can be clearly observed that, for a fixed fill distance, decreases in the separation alone will rapidly decrease the minimum eigenvalue and, thus, increase the condition number of the matrix and potentially lead to poor numerical results. 

\begin{figure}
    \centering
    \includegraphics[width = \textwidth,trim={0 0.1cm 0 0.1cm},clip]{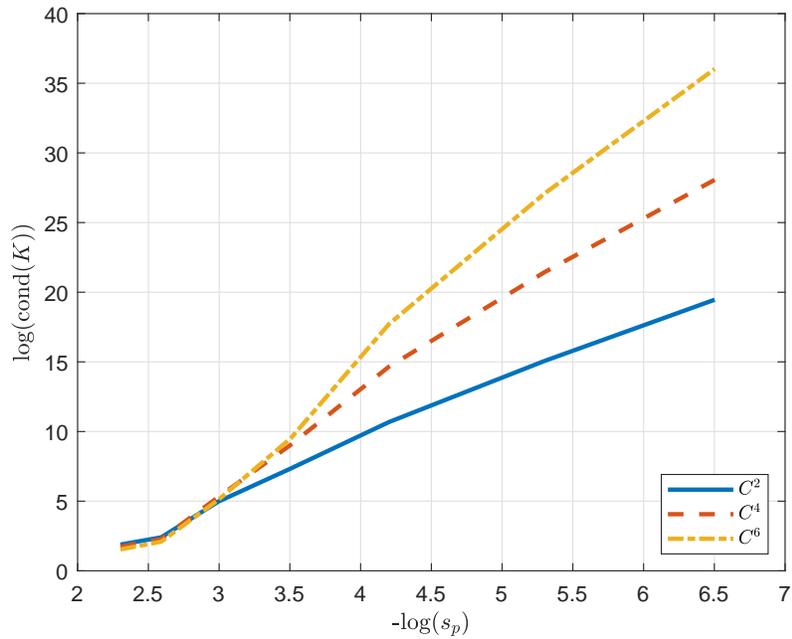}
    \caption{Three cases of the condition number of the kernel matrix, $\bm{K}_M$, versus the minimum separation distance of the collected centers from the simulated data. Each curve corresponds to the compactly supported Wendland functions of smoothness class $C^2$,$C^4$ and $C^6$. In this example, centers are quasi-uniformly distributed based on some predetermined spacing, $s_p$, which is of the same order as the fill distance, $h_{\Xi_M,Q}$, and the separation, $s_{\Xi_M}$. The number of centers ranged from  $M = 2$ for the largest separation to $64$ for the smallest. Note that as the separation distance decreases the condition number increases exponentially. Additionally, the condition number of the kernel matrix associated with higher order kernels is larger for the same separation distances.}
    \label{fig:uniformSeparationSmoothnessTest}
\end{figure}

\begin{figure}
    \centering
    \includegraphics[width = \textwidth,trim={0 0.1cm 0 0.1cm},clip]{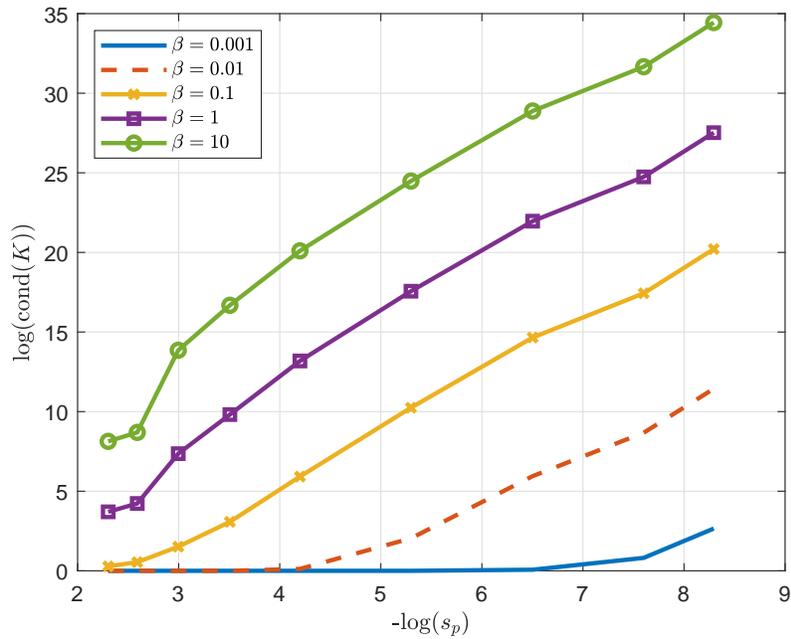}
    \caption{This figures demonstrates the relationship between the fill distance and the condition number for the Matern-Sobolev kernel, $\knl^{MS}_2$, for various values of the hyperparameter, $\beta$. In this example, centers are quasi-uniformly distributed based on some predetermined spacing, $s_p$, which is of the same order as the fill distance, $h_{\Xi_M,Q}$, and the separation, $s_{\Xi_M}$. Like the previous study, the number of centers,  $M$, ranged from $2$ to $64$ based on the separation distance. The smaller spread of the kernel function corresponding to smaller $\beta$ lowers the condition number for estimates of similar fill distances. However, the rate of increase given by the slope is determined by the kernel type as is, therefore, relatively the same for each curve in this example.}
    \label{fig:uniformSeparationBetaTest}
\end{figure}

\begin{figure}
    \centering
    \includegraphics[width = \textwidth,trim={0 0.1cm 0 0.1cm},clip]{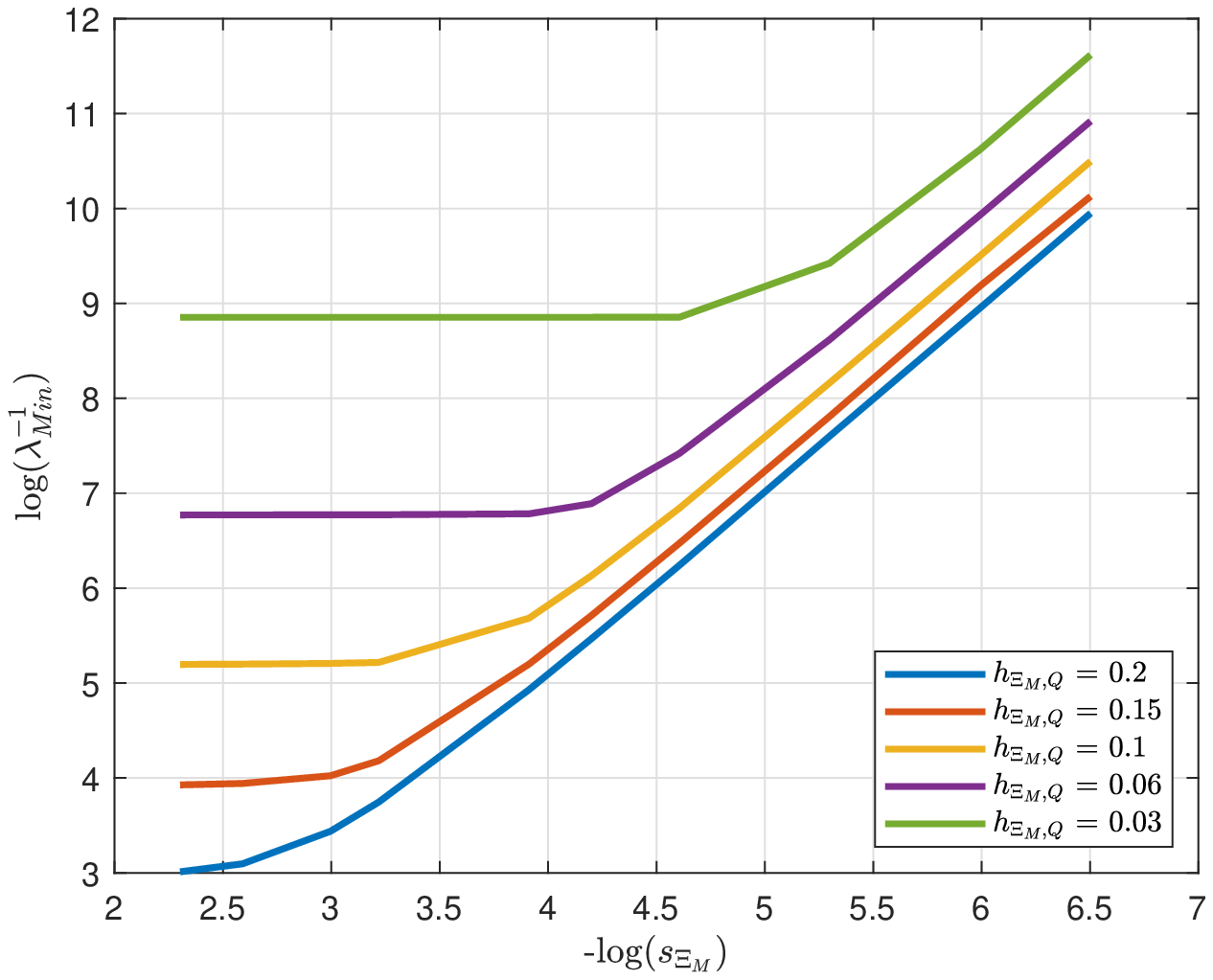}
    \caption{A figure illustrating the behavior of the minimum eigenvalue of the Kernel Matrix $K$ given the minimum separation distance of the collected centers from the simulate data. For each curve, the separation, $s_{\Xi_M}$ is decreased between only two centers while keeping a constant fill distance $h_{\Xi_M.Q}$. The number of centers per curve $M$ ranged from $7$ to $37$ based on the specified fill distance between the centers.}
    \label{fig:nonuniformSeparation}
\end{figure}

\subsection{Example: Human Kinematics Study}
Our second example uses three-dimensional motion capture data  collected in \cite{fukuchi2018public} during  a human  walking  experiment on  a treadmill. From their experiment, we collected 5000 marker coordinates   relative to a fixed inertial frame defined by the camera's fixed position. For this example, a small candidate kinematic model is defined from the full collection of experimental trajectories. The marker coordinates of the hip, knee, and ankle in the full data set are projected to the sagittal plane that divides the left and right half of the body see Figure \ref{fig:kinematicFigure}. With this projection, the joint angle  $\theta^{(1)}$ roughly corresponds to hip flexion, and it is  calculated from the projected vector $v^{(1)}$ that connects the hip to knee from the inertial $b^{(1)}$ vector in the plane. The knee flexion angle $\theta^{(2)}$ is  calculated by taking the dot product between the vector $v_1$ and the vector connecting the knee to the ankle $v^2$ in the sagittal plane. We then work with the assumption that samples of these joint angles come from some configuration manifold $Q$. Here the indexes $\{k_i\}_{i=1}^M$ and thus the sample pairs $\{(\xi_{M,i},\upsilon_{M,i})\}_{i=1}^M$ were chosen so that there was sufficient separation distance of at least 0.5 radians between the kernel centers, $\Xi_M$.  For illustrative purposes, we chose to examine the kinematic maps $G^{(1)}$ and $G^{(2)}$ from  the joint angles $\theta^{(1)}$ and $\theta^{(2)}$ to the ankle  $y^{(1)}$ and $y^{(2)}$ associated with the $b^{(1)}$ and $b^{(2)}$ vectors in the body-fixed frame. By choosing the outputs  $\{y^{(1)}_{k_i+1}\}_{i=1}^M= \{\upsilon^{(1)}_{M,i}\}_{i=1}^M =\Upsilon_M $ to be the ankle coordinates associated with the $b^{(1)}$ vector, we can visualize the Koopman operator acting on function $G^{(1)}:\theta^{(1)} \times \theta^{(2)} \to y^{(1)}$ as shown in Figure \ref{fig:kinematicEstimate}a. Similarly, we can define the outputs, $\Upsilon_M$ to be  $\{y^{(2)}_{k_i+1}\}_{i=1}^M$ to estimate the action of the Koopman operator on the function $G^{(2)}:\theta^{(1)} \times \theta^{(2)} \to y^(2)$ mapping to coordinates associated with the $b^{(2)}$ vector as is shown in Figure \ref{fig:kinematicEstimate}b.

Figure \ref{fig:kinematicEstimate} was generated using $M=4249$ samples and the Matern-Sobolev kernel with $\beta =2$. Both estimates are generated over the same submanifold $Q$ of the input space with coordinates $\theta^{(1)}$ and $\theta^{(2)}$. However, from the figure we can see that the Koopman operator acting on $y^{(1)}$, which corresponds to the height of the ankle from the hip, remains relatively constant throughout the motion while the Koopman operator acting on $y^(2)$, which associated with forward movement relative to the body, undergoes significant changes over the submanifold $Q$.

\begin{figure}[htp]
    \centering
     \includegraphics[width = \textwidth]{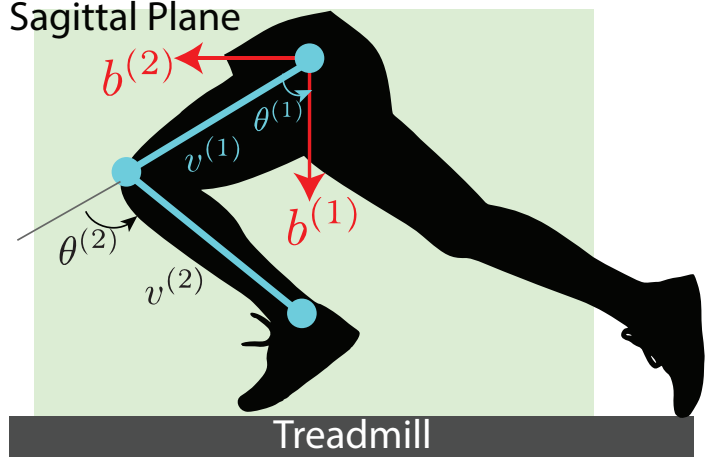}
    \caption{An illustration of the inputs $\theta^{(1)}$ and $\theta^{(2)}$, which roughly correspond to hip and knee flexion respectively.  These input variables can be mapped to measured marker coordinates placed on joints such as the knee or ankle.}
    \label{fig:kinematicFigure}
\end{figure}

\begin{figure}[htp]
\begin{subfigure}{\textwidth}
 \centering
     \includegraphics[width = \textwidth,trim={0 0.1cm 0 1cm},clip]{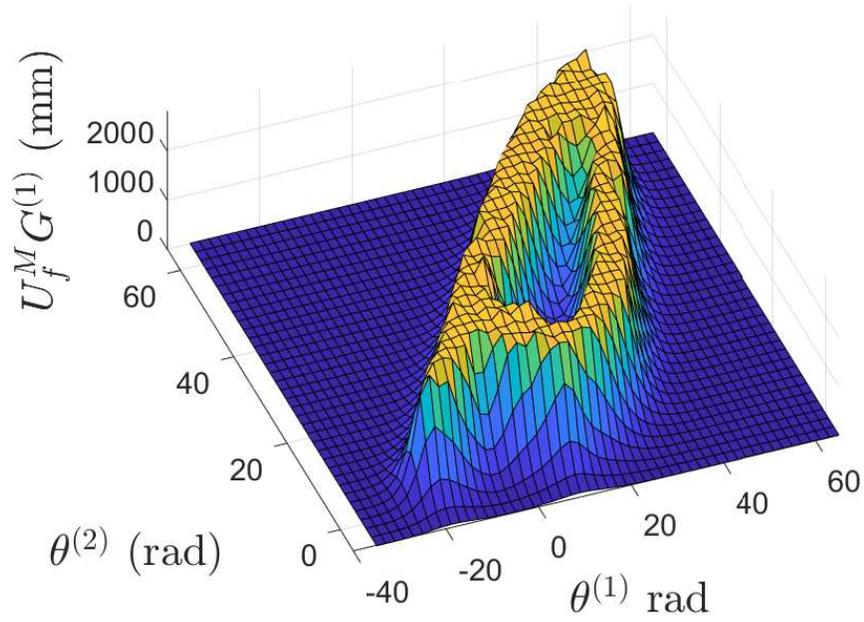}
    \caption*{(a)}
    
\end{subfigure}
\begin{subfigure}{\textwidth}
\centering
    \includegraphics[width = \textwidth,trim={0 0.1cm 0 1cm},clip]{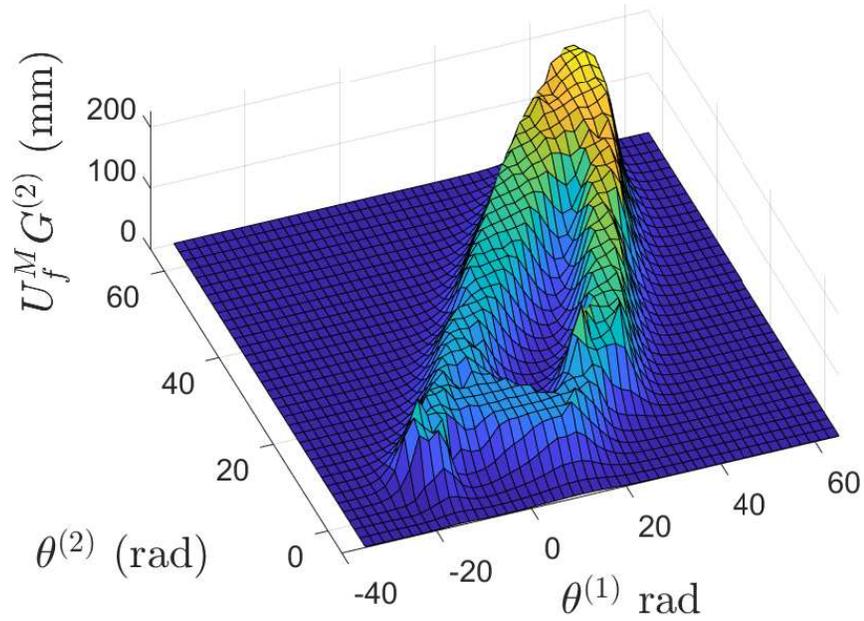}\caption*{ (b)}

\end{subfigure}
\caption{The estimates of the coordinates of the ankle associated with (a) the $b^{(1)}$ vector and  (b) the $b^{(2)}$ vector over the submanifold $Q$ of the input space given by the input coordinates $\theta^{(1)}$ and $\theta^{(2)}$. Here $5000$ samples were initially examined. Of those, $4249$ were selected that had sufficient separation of distance greater than $0.5$ radians between inputs.}
\label{fig:kinematicEstimate}
\end{figure}
\begin{figure}[htp]
   
\end{figure}

\section{Conclusion}
This paper has derived methods for estimation of forward kinematics of animal motions. We have derived a closed form expression for the empirical risk minimizer in a pullback RKHS space that depends on the unknown dynamics. We also have shown that approximations of the unknown function  $G:Q\rightarrow Y$ can be constructed in terms of a data-dependent approximation of the the Koopman operator. For the kernel based method, we have derived error bounds based on scattered samples and kernel dependent bases defined over $Q$. Finally, we have presented numerical and experimental results to better illustrate the qualitative behavior of the algorithms. 
\section*{Data Availability}
The datasets generated during and/or analysed during the current study are available from the corresponding author on reasonable request
\section*{Appendix}
The proof of convergence for the approximation of the Koopman operator in Section \ref{sec:data_driven} uses a version of the many zeros theorem that holds over compact, smooth, regularly embedded submanifolds $S\subset \RR^d$.
\begin{theorem}[\cite{fuselier}Lemma 10]
\label{th:manyz}
Let $Q$ be a $k$ dimensional smooth manifold, $s\in \RR$ with $s>k/2$, and $m\in \NN$ satisfy $0\leq m \leq \lceil s \rceil-1$. There there are constants $C,H$ such that if $h_{\Xi_M,Q}\leq H$ and $f\in W^{s,2}(Q)$ satisfies $f|_{\Xi_M}=0$, then $$
|f|_{W^{m,2}(Q)} \leq C h^{s-m}_{\Xi_M} |f|_{W^{s,2}(Q)}. 
$$
\end{theorem}
To apply this theorem in applications, we will need to choose the kernel $\knl$ that is defined on all of $X$ in such a way that that the space of restrictions $\RH$ to $Q$ is equivalent to a Sobolev space over $Q$. This can be achieved using the strategy outline on page 1759 of \cite{fuselier}. The kernel $\knl$ on $X\times X$ is said to have algebraic decay if its Fourier transform $\hat{\knl}$ has algebraic decay in the sense that
$ 
\hat{\knl}(\xi)\sim  (1+\|\xi\|_2^2)^{-\tau}
$ 
for $\tau>d/2$. If the kernel satisfies this algebraic decay condition, then it follows that $H=W^{\tau,2}(\RR^d)$. We then can use the following theorem, again from \cite{fuselier}.
\begin{theorem}[\cite{fuselier} Theorem 5]
If the kernel $\knl$ that induces $H$ satisfies the algebraic growth condition with exponent $\tau$, then we have 
$
\RH=T(H)=W^{\tau-(d-k)/2,2}(Q).
$
\end{theorem}
\bibliographystyle{unsrt}

\bibliography{CDCPaperRef}
%
%%  overflow
%
\clearpage
%   cut for now 
%

\end{document}